%% file: article.tex
\documentclass[review,hidelinks,onefignum,onetabnum]{siamart220329}
\input{shared}
\input{math_commands.tex}
\ifpdf
\hypersetup{
  pdftitle={Paired Autoencoders for Inverse Problems},
  pdfauthor={Emma Hart, Julianne Chung, and Matthias Chung}
}
\fi
\begin{document}
\nolinenumbers
\maketitle
\begin{abstract}
In this work, we describe a new data-driven approach for inverse problems that exploits technologies from machine learning, in particular autoencoder network structures.  We consider a paired autoencoder framework, where two autoencoders are used to efficiently represent the input and target spaces separately and optimal mappings are learned between latent spaces, thus enabling forward and inverse surrogate mappings. We focus on interpretations using Bayes risk and empirical Bayes risk minimization, and we provide various theoretical results and connections to existing works on low-rank matrix approximations. Similar to end-to-end approaches, our paired approach creates a surrogate model for forward propagation and regularized inversion. However, our approach outperforms existing approaches in scenarios where training data for unsupervised learning are readily available but training pairs for supervised learning are scarce.  Furthermore, we show that cheaply computable evaluation metrics are available through this framework and can be used to predict whether the solution for a new sample should be predicted well.
\end{abstract}



\section{Introduction}
Inverse problems are ubiquitous, with scientific applications in biomedical and geophysical imaging \cite{hansen2010discrete}, atmospheric modeling \cite{enting2002inverse}, and data assimilation \cite{sanz2023inverse}, to name a few.  
Let us consider the following problem,
\begin{equation}
\label{eq:inverseproblem}
\bfb = A(\bfx)+\bfeps,
\end{equation}
where $\bfb\in\bbR^q$ contains observations, $A:\bbR^n\to\bbR^q$ represents a forward parameter-to-observation process, $\bfx\in\bbR^n$ contains parameters of interest, and $\bfvarepsilon\in\bbR^q$ represents noise.  
The inverse problem aims to determine the parameters $\bfx$, given the observed measurements $\bfb$ and knowledge of $A$. 

There are a multitude of computational challenges that arise when solving inverse problems, especially those with large-dimensional parameters and observation sets (i.e., $q$ and/or $n$ are large).  First, it is often the case that an accurate mathematical model $A$ is needed to represent the forward model, but only a surrogate model or approximation is available for computation.
Second, regularization or prior knowledge is needed to compute reliable solutions due to ill-posedness of the underlying problem, but selecting appropriate priors and regularization parameters is often a heuristic process that requires expert knowledge and application expertise. Third, sophisticated computational methods are needed to compute solutions, and these are often iterative optimization procedures that require multiple forward and adjoint solves. Another layer of complexity (but also opportunity) when solving inverse problems is determining how to incorporate any available data (e.g., training data or large scientific datasets) to improve or support the inversion process.  

In this work, we describe a new computational framework for inverse problems in which we couple machine learning tools for dimensionality reduction with surrogate models in a reduced space for the forward and inverse mappings. More specifically, consider an inverse mapping $\widebar \Phi: \bfb \to \bfx$ that maps observations $\bfb$ to parameters $\bfx$. 
Our approach can learn the inverse mapping $\widebar \Phi$ (as well as the forward mapping $A$) by training autoencoders on inputs $\bfb$ \emph{and} targets $\bfx$, while also discovering a mapping between the reduced latent spaces.  The proposed framework is rooted in theory for Bayes risk and empirical Bayes risk minimization, making the approach versatile and optimal under certain assumptions. Moreover, the paired autoencoder framework can be used to address various challenges related to both the forward and the inverse problem.  
 
Previous works in this field have used neural networks for full inversion (surrogate modeling) \cite{kulkarni2016reconnet}, regularization \cite{afkham2021learning,li2020nett}, uncertainty quantification \cite{goh2019solving,lan2022scaling}, and more \cite{bai2020deep,lucas2018using}. 
Similar to ideas in operator learning \cite{kovachki2023neural} and reduced order modeling \cite{lee2020model}, where surrogate operators and reduced models are learned in a latent space, we aim to utilize latent structures to learn mappings between input and target spaces for inverse problems.  We propose to use autoencoders to learn latent representations, and then pair the autoencoders by learning mappings between their latent spaces.  
If one considers a fully linear approach for the autoencoders and mappings between latent spaces, our proposed approach reduces to inversions based on principal component analysis (PCA) \cite{benner2015survey,antoulas2020interpolatory}. 
Our proposed framework has an inherent regularizing property, which is beneficial since there is no need to select regularization parameters or utilize expensive algorithms to solve because we can efficiently represent the entire regularized inverse map.

This work carries similarities to methods that have been developed for image-to-image translation to predict, for instance, geophysical properties from measurements \cite{feng2022intriguing, feng2024auto, chung2024paired, wang2024wave, feng2024hidden, piening2024paired}. For example, an approach has been developed in \cite{feng2022intriguing} that learns only the decoder from input/target sample pairs, where the encoder is predefined with a fixed kernel (e.g., sine, Fourier, or Gaussian). A linear mapping is established and fixed between the transformed input/target spaces, allowing the decoder to be optimized via stochastic gradient descent. Another method leverages masked autoencoders independently in the input and target spaces, which are subsequently fixed and connected by a linear mapping between their latent spaces \cite{feng2024auto}. A likelihood-free inverse surrogate is considered in \cite{chung2024paired}, where two coupled autoencoders are trained jointly using a combined loss function that incorporates both inversion and autoencoder losses, thereby facilitating a shared latent space. Two metrics -- the relative residual estimate and the recovered model autoencoder -- were used to assess whether a new sample aligns with the network’s training distribution.  A similar framework was also used for image generation, where two autoencoders are trained and a joint diffusion process is used in the latent space to generate high-quality input/target pairs that satisfy physical constraints imposed by the wave equation \cite{wang2024wave}.  Interestingly, even in high-dimensional nonlinear problems, the mappings between latent spaces are empirically found to be linear \cite{feng2022intriguing, feng2024hidden}.

Our work builds on these papers in some key ways; we introduce and explore the paired autoencoder methodology as a general framework for inverse problems.  We train the autoencoders in parallel and learn a linear map between their latent spaces, like \cite{feng2024auto} and unlike \cite{chung2024paired}.  Contrary to previous works, we investigate the paired autoencoder framework via a Bayes risk and empirical Bayes risk minimization perspective, thereby enabling theoretical results for simplified cases, as well as drawing connections to traditional techniques from scientific computing.  

\paragraph{Overview of contributions}  In this work, we describe a new data-driven framework for inverse problems called \emph{Paired Autoencoders for Inference and Regularization (PAIR)}.  The PAIR framework exploits autoencoder networks for both the input and target spaces separately and uses an optimal mapping between latent spaces. This approach confers several advantages:
\begin{itemize}
\item Compared to end-to-end approaches, i.e., networks mapping $\bfb \mapsto \bfx$, the PAIR network is superior for problems with many training samples but few input-target pairs.  Moreover, since the PAIR approach decouples the model and the dimension reduction processes, autoencoders can have latent spaces with different dimensions, input and target datasets can be of different sizes, and training via self-supervised learning techniques can be done independently and in parallel.
 The different dimensions of the latent spaces provide flexibility in representation (e.g., so that compression can be done at different levels),
and the propagation of uncertainty associated with model uncertainties, noise in the data, and data compression can also be separated. 
\item By building on connections to Bayes risk and empirical Bayes risk minimization problems, we provide new theoretical results for linear autoencoders and linear latent space mappings. In addition to providing theory for linear autoencoders related to optimal low-rank matrix approximation, we provide theory for optimal linear mappings between latent spaces.  These provide interpretations for the PAIR forward and inverse surrogates and are relevant for both linear and nonlinear autoencoders.
\item The PAIR framework can be applied to general inverse problems and can appeal to a broad scientific computing audience. Although we focus on the Bayes risk minimization interpretation resulting in a minimization of the expected loss in the 2-norm, other error metrics could be used.  Moreover, the PAIR framework provides new insights into reduced-order and surrogate modeling by using the perspective of coupled autoencoders and neural networks in general.
\end{itemize}

An overview of the paper is as follows.  In \Cref{sec:background} we provide background on regularization for solving problems such as \Cref{eq:inverseproblem}, encoder/decoder networks, and autoencoders in particular.  Then in  \Cref{sec:PAIRframework}, we present the PAIR framework, focusing on the theory for linear inverse problems with linear autoencoders and linear mappings between latent spaces. Numerical experiments provided in \Cref{sec:numerics} illustrate the benefits of our approach for various imaging examples.  We remark that although we focus on examples from imaging, the work is broadly applicable since similar problems arise in different scientific applications. 
Conclusions and future work are provided in \Cref{sec:conclusions}.

\section{Background}
\label{sec:background}
Solving inverse problems such as \Cref{eq:inverseproblem} requires a method that can approximate parameters $\bfx$ from noisy data $\bfb$.  This is a difficult and often ill-posed problem in applications; the solution may not be unique, and it may not depend continuously on observed data \cite{hadamard1923lectures}. The direct computation of an inverse (or pseudo-inverse) solution can be computationally burdensome, and such solutions may be highly susceptible to noise corruption \cite{hansen2010discrete}.

\paragraph{Regularization}
Usually, regularization stabilizes the inversion process, creating a solution that is not so sensitive to errors in the observations.  Standard forms of regularization for large-scale inverse problems include variational and iterative regularization \cite{chung2021computational}, which confer different advantages and disadvantages.  Variational approaches can easily accommodate constraints and priors, but choosing a suitable regularization parameter is often computationally expensive.  Iterative regularization approaches require no such parameter, but a good stopping criterion is essential and priors and constraints are not as easy to incorporate.  A shared disadvantage is that both require full model evaluations (and adjoints) during optimization.  Data-driven regularization approaches address these weaknesses by leveraging available data, models, and domain-specific knowledge.  Examples include dictionary learning, bilevel learning, Markov random field-type regularizers, and, increasingly, deep neural networks \cite[Section 5]{arridge2019solving}.

\paragraph{Encoder-Decoder networks}
One popular choice is to implement deep neural networks with Encoder-Decoder (ED) architectures, which consist of two parts (the encoder and the decoder) with a hidden layer with $\bfz \in \bbR^r$ that describes a representation of the encoded/decoded signal.  Typically, the encoder network is comprised of layers that successively reduce the dimension of the input vector, to construct a low-dimensional vector $\bfz$ referred to as the \emph{latent} vector. Then, the decoder network increases the dimension through the network to match the output dimension.  More specifically, the two main parts can be defined using parametric mappings.
\begin{itemize}
    \item The encoder network $e:\bbR^q \times \bbR^{p^{\rm e}}\to \bbR^r$ with network parameters $\bftheta^{\rm e} \in \bbR^{p^{\rm e}}$ maps an input $\bfb$ to the hidden layer or latent variables $\bfz$, i.e.,
\begin{equation}
    \bfz = e(\bfb;\,\bftheta^{\rm e}).
\end{equation}
\item The decoder network $d:\bbR^r\times \bbR^{p^{\rm d}} \to \bbR^n$ with network parameters $\bftheta^{\rm d} \in \bbR^{p^{\rm d}}$ maps the latent variables to the output $\bfx$, i.e.,
\begin{equation}
    \bfx = d(\bfz;\, \bftheta^{\rm d}).
\end{equation}
\end{itemize}
The entire ED network then takes the form $\Phi_{\rm ed}(\bfb; \, \bftheta^{\rm e}, \bftheta^{\rm d}) := d(e(\bfb; \, \bftheta^{\rm e})\, ;\, \bftheta^{\rm d})$. See \Cref{fig:EDnetwork} for a visual representation of an ED network for mapping $\bfb$ to $\bfx$.

\begin{figure}[bthp]
    \centering
    \begin{tikzpicture}
    	\node[fill=matlab5!50, minimum width=0.5cm, minimum height=3.5cm] (x) at (0,0) {$\bfb$};
    	
    	\draw[fill=matlab6!50,draw=none] ([xshift=0.5cm]x.north east) -- ([xshift=2.5cm,yshift=0.5cm]x.east) -- ([xshift=2.5cm,yshift=-0.5cm]x.east) -- ([xshift=0.5cm]x.south east) -- cycle; 
    	\node at (1.75,0) {$\begin{matrix}\mbox{encoder} \\ e \end{matrix}$};
    	
    	\node[fill=matlab3!50, minimum width=0.5cm, minimum height=1.0cm] (Zx) at (3.5cm,0) {$\bfz$};
    
     	\draw[fill=matlab4!50,draw=none] ([xshift=0.5cm]Zx.north east) -- ([xshift=2.5cm,yshift=0.75cm]Zx.north east) -- ([xshift=2.5cm,yshift=-0.75cm]Zx.south east) -- ([xshift=0.5cm]Zx.south east) -- cycle;
    	\node at (5.25,0) {$\begin{matrix}\mbox{decoder} \\ d \end{matrix}$};
    
    	\node[fill=matlab1!50, minimum width=0.5cm, minimum height=2.5cm] (B) at (7,0) {$\bfx$}; 
    \end{tikzpicture}
    \caption{End-to-end encoder-decoder network for inversion. The network is mapping input vector $\bfb$ to output vector $\bfx$.  The encoder maps the input vector $\bfb$ to the latent variable $\bfz$ and the decoder maps this latent variable $\bfz$ to the output $\bfx$.}
    \label{fig:EDnetwork}
\end{figure}
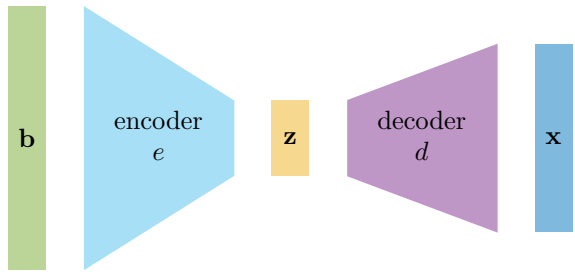

\bigskip 

These parameterized ED networks can be surrogates to unidentified or computationally demanding mappings $\widebar\Phi: \bfb \to \bfx$. To construct a direct end-to-end inversion network requires network parameters $\bftheta^{\rm e}$ and $\bftheta^{\rm d}$ that approximate such a mapping, i.e.,  $\Phi_{\rm ed}(\bfb; \, \bftheta^{\rm e}, \bftheta^{\rm d})   \approx \widebar\Phi(\bfb)$.  In a data-driven approach, the network parameters are calibrated by minimizing some \emph{loss function}, given a sufficient amount of representative and labeled input-target pairs $\left\{(\bfb_j,\bfx_j)\right\}_{j = 1}^J$, e.g.,
\begin{equation} \label{eq:loss}
    \min_{\bftheta^{\rm e}, \bftheta^{\rm d}} \quad \tfrac{1}{J} \sum_{j = 1}^J \calL(\Phi_{\rm ed}(\bfb_j;\bftheta^{\rm e}, \bftheta^{\rm d}); \bfx_j),
\end{equation}
where $\calL(\mdot;\mdot)$ measures the discrepancy between $\bfx_j$  and $\Phi_{\rm ed}(\bfb_j; \, \bftheta^{\rm e}, \bftheta^{\rm d})$ for a particular set of parameters $\bftheta^{\rm e}, \bftheta^{\rm d}$.
Depending on the data at hand and the selected network architecture, regularization on the network weights may be imposed in \Cref{eq:loss}.

\paragraph{Autoencoder networks} An \emph{autoencoder} is a special type of ED network that maps an input to itself, for example, $\Phi_{\rm ae}^\bfb: \bbR^q \to \bbR^q$ with $\Phi_{\rm ae}^\bfb(\bfb; \, \bftheta^{\rm e}_\bfb, \bftheta^{\rm d}_\bfb) \approx \bfb$, \cite{goodfellow2016deep}. The encoder represents the input in a lower dimensional space while the decoder attempts to reconstruct the input from the low dimensional representation $\bfz = e_\bfb(\bfb; \bftheta^{\rm e}_\bfb)$. In the training of autoencoders, the flow of data through the low-dimensional latent space forces the autoencoder to find a low-dimensional representation for a set of data, during which redundant data is eliminated. Hence, such networks are classically used in dimensionality and noise reduction, as well as data compression applications, see e.g., \cite{hinton2006reducing,salakhutdinov2007restricted,torralba2008small,wang2016auto,chung2024sparse}. The network parameters $\bftheta^{\rm e}_\bfb, \bftheta^{\rm d}_\bfb$ of an autoencoder $\Phi_{\rm ae}^\bfb$ are found by minimizing a loss function given \emph{unlabeled} data $\{\bfb_j\}_{j = 1}^Q$, which is considered an \emph{unsupervised} or \emph{self-supervised} learning task. In an identical fashion an autoencoder for $\bfx$ can be considered $\Phi_{\rm ae}^\bfx: \bbR^n \to \bbR^n$ with $\Phi_{\rm ae}^\bfx(\bfx; \, \bftheta^{\rm e}_\bfx, \bftheta^{\rm d}_\bfx) \approx \bfx$. The reduced order modeling community has embraced autoencoders as a tool for model reduction, especially as a nonlinear generalization within the proper orthogonal decomposition framework \cite{gonzalez2018deep,pichi2024graph,lee2020model,benner2015survey,antoulas2020interpolatory}.  

\section{PAIR framework}
\label{sec:PAIRframework}
Our proposed PAIR approach utilizes a reduced-order network architecture that combines two separate autoencoders for the inputs and the targets with optimal mappings between their respective latent spaces. See \Cref{fig:PAIR} for an illustration.
The autoencoders $\Phi_{\rm ae}^\bfx = d_\bfx \circ e_\bfx$ and $\Phi_{\rm ae}^\bfb = d_\bfb \circ e_\bfb$  provide dimensionality reduction for the input and target parameter spaces that when combined with the mappings between latent spaces, $m\colon \bbR^{r_\bfx} \to \bbR^{r_\bfb}$ and $m^\dagger\colon \bbR^{r_\bfb} \to \bbR^{r_\bfx}$, provide a data-driven surrogate for the forward simulation and the inversion of the system via the PAIR surrogates
\begin{equation}
    P = d_\bfb \circ m \circ e_\bfx \quad\quad \text{and} \quad\quad P^\dagger  = d_\bfx \circ m^\dagger \circ e_\bfb .
\end{equation}
The autoencoders $\Phi_{\text{ae}}^\bfx$ and $\Phi_\text{ae}^\bfb$ are learned independently on $\bfx$ \emph{and} $\bfb$, correspondingly, and then optimal mappings $m,m^\dagger$ between the reduced latent spaces $\bbR^{r_\bfx}$ and $\bbR^{r_\bfb}$ are obtained. 

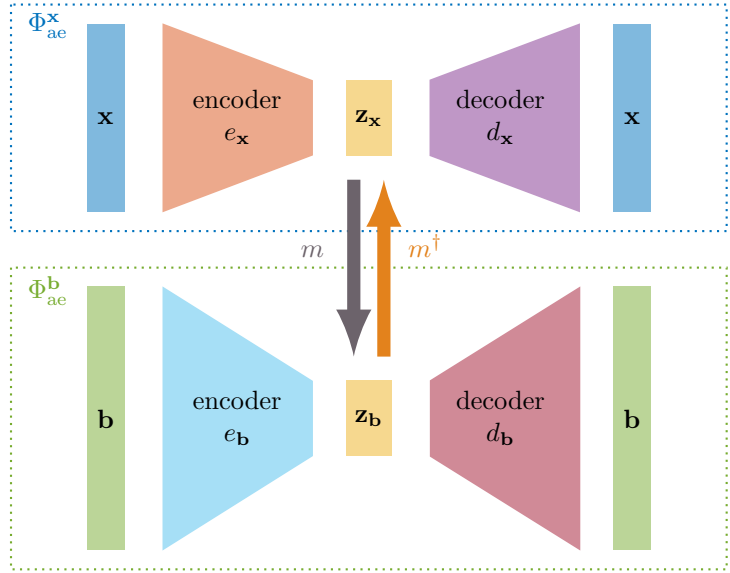
\begin{figure}[bthp]
    \centering
    \begin{tikzpicture}    
	\node[fill=matlab1!50, minimum width=0.5cm, minimum height=2.5cm] (x) at (0,0) {$\bfx$};
	
	\draw[fill=matlab2!50,draw=none] ([xshift=0.5cm]x.north east) -- ([xshift=2.5cm,yshift=0.5cm]x.east) -- ([xshift=2.5cm,yshift=-0.5cm]x.east) -- ([xshift=0.5cm]x.south east) -- cycle; 
	\node at (1.75,0) {$\begin{matrix}\mbox{encoder} \\ e_{\bfx} \end{matrix}$};
	
	\node[fill=matlab3!50, minimum width=0.5cm, minimum height=1.0cm] (Zx) at (3.5cm,0) {$\bfz_{\bfx}$};
	
	\draw[fill=matlab4!50,draw=none] ([xshift=0.5cm]Zx.north east) -- ([xshift=2.5cm,yshift=0.75cm]Zx.north east) -- ([xshift=2.5cm,yshift=-0.75cm]Zx.south east) -- ([xshift=0.5cm]Zx.south east) -- cycle;
	\node at (5.25,0) {$\begin{matrix}\mbox{decoder} \\ d_{\bfx} \end{matrix}$};
	
	\node[fill=matlab1!50, minimum width=0.5cm, minimum height=2.5cm] (X) at (7,0) {$\bfx$};

 	\node[fill=matlab5!50, minimum width=0.5cm, minimum height=3.5cm] (b) at (0,-4) {$\bfb$};
	
	\draw[fill=matlab6!50,draw=none] ([xshift=0.5cm]b.north east) -- ([xshift=2.5cm,yshift=0.5cm]b.east) -- ([xshift=2.5cm,yshift=-0.5cm]b.east) -- ([xshift=0.5cm]b.south east) -- cycle; 
	\node at (1.75,-4) {$\begin{matrix}\mbox{encoder} \\ e_{\bfb} \end{matrix}$};
	
	\node[fill=matlab3!50, minimum width=0.5cm, minimum height=1.0cm] (Zb) at (3.5cm,-4) {$\bfz_{\bfb}$};
	
	\draw[fill=matlab7!50,draw=none] ([xshift=0.5cm]Zb.north east) -- ([xshift=2.5cm,yshift=1.25cm]Zb.north east) -- ([xshift=2.5cm,yshift=-1.25cm]Zb.south east) -- ([xshift=0.5cm]Zb.south east) -- cycle;
	\node at (5.25,-4) {$\begin{matrix}\mbox{decoder} \\ d_{\bfb} \end{matrix}$};
	
	\node[fill=matlab5!50, minimum width=0.5cm, minimum height=3.5cm] (B) at (7,-4) {$\bfb$};
\draw[matlab1!100, dotted, thick] ([xshift=-1cm, yshift=0.25cm]x.north west) rectangle ([xshift=1cm, yshift=-0.25cm]X.south east) node[left, very near start, xshift=-0.3cm, yshift=0.1cm] {$\Phi_{\rm ae}^{\bfx}$};
\draw[matlab5!100, dotted, thick] ([xshift=-1cm, yshift=0.25cm]b.north west) rectangle ([xshift=1cm, yshift=-0.25cm]B.south east) node[left, very near start, xshift=-0.3cm, yshift=0.2cm] {$\Phi_{\rm ae}^{\bfb}$};

\begin{scope}[-latex,shorten >=9pt,shorten <=9pt,line width=5pt]
    \draw[matlab1!50!matlab2] ([xshift=-0.2cm]Zx.south) to ([xshift=-0.2cm]Zb.north);
    \draw[matlab2!50!matlab3] ([xshift=0.2cm]Zb.north) to ([xshift=0.2cm]Zx.south);
\end{scope}
\node[matlab1!50!matlab2] (f_forward) at (2.75,-1.77) {$m$};
\node[matlab2!50!matlab3] (f_inverse) at (4.25,-1.7) {$m^\dagger$};
\end{tikzpicture}
    \label{fig:PAIR}
    \caption{PAIR network mapping.  Two autoencoders are used to compress vectors $\bfx$ and $\bfb$ (on the top and bottom respectively), where the corresponding latent spaces are represented using variables $\bfz_\bfx$ and $\bfz_\bfb$. 
    Mappings between latent spaces are denoted as $m$ and $m^\dagger$. The PAIR network provides both a data-driven inverse mapping $d_\bfx \circ m^\dagger \circ e_\bfb$ and a data-driven forward surrogate approximation $d_\bfb \circ m \circ e_\bfx$.}
\end{figure}

In \Cref{sub:theory_autoencoder}, we develop theory for autoencoders where the encoder and decoder are both linear mappings.  We show that finding the linear autoencoder reduces to a rank-constrained optimization problem in a Bayes risk minimization interpretation.  In an empirical Bayes risk minimization framework, where training data are given, we provide closed-form solutions to the optimization problem, since the above framework reduces to classical approaches based on the singular value decomposition (SVD). The connection between linear autoencoder and SVD/PCA has been discussed in various works, for instance, \cite{bourlard1988auto,baldi1989neural,plaut2018principal,bao2020regularized}. However, these works do not target the more general Bayes risk minimization interpretation as we do here. 
Then, in \Cref{sub:theory_linmap} we focus on linear problems where $\bfA \in \bbR^{q \times n}$. We assume that both autoencoders (for $\bfx$ and $\bfb$) are linear mappings, and we develop theory for linear mappings $m$ and  $m^\dagger$ between latent spaces in both the Bayes risk and empirical Bayes risk settings.  We conclude the section by discussing surrogates provided by the fully linear PAIR framework, with connections to classical linear algebra results.

\subsection{Theory for linear autoencoders} \label{sub:theory_autoencoder}
Let us consider the case of a (single) \emph{linear} autoencoder network, in which we map an input $\bfx\in\bbR^n$ to itself, by first \textit{encoding} the signal $\bfx$ with an encoder $e:\bbR^n \to \bbR^r$ to a latent space with $\bfz = e(\bfx;\vec{\bfE}) = \bfE\bfx$ and $\bfE\in \bbR^{r \times n}$, and then \textit{decoding} $\bfz$ with a decoder  $d:\bbR^r \to \bbR^n$ back to itself with $d(\bfz; \vec{\bfD}) = \bfD\bfz$ and $\bfD\in \bbR^{n \times r}$.  Subscripts on the encoder and decoder are omitted for clarity in this section. We assume the latent space dimension $r$ to be arbitrary but fixed with $0<r<n$. 
With these assumptions, the autoencoder is given by
\begin{equation}
\label{eq:linautoencoder}
    \Phi_{\rm ae}^\bfx(\bfx; \, \vec{\bfE}, \vec{\bfD}) = \bfD\bfE\bfx \equiv \bfY \bfx,
\end{equation}
where $\bfY = \bfD \bfE$ and $\vec{\bfE}$ is the vectorization of matrix $\bfE$.
Thus, linear autoencoders aim to find matrices $\bfE$ and $\bfD$ such that the mapping $\bfD\bfE\bfx$ is close to $\bfx$. 

We assume $\bfx$ is a realization of a random variable $X$ with some underlying probability distribution. We use a Bayes risk minimization interpretation to show that the problem of estimating linear autoencoders reduces to a low-rank minimization problem with an expected value loss function.  Then, we describe an empirical Bayes risk minimization approach that uses sampled data.

\subsubsection{A Bayes risk minimization interpretation}\label{sub:ae_Bayes}
Let $X$ be a random variable with a given probability distribution, and consider a linear autoencoder \cref{eq:linautoencoder}. The goal is to determine matrices $\widetilde\bfE$ and $\widetilde\bfD$ such that a predefined distance measure between $\bfD\bfE X$ and $X$ is minimized. One possibility is to minimize the expected loss,
\begin{equation}
\label{eq:expectedloss}
(\widetilde\bfE,\widetilde \bfD) = \argmin_{\bfE,\bfD} \ \bbE \norm[2]{\bfD\bfE X - X}^2.
\end{equation}
We can simplify this to 
\begin{equation}
\label{eq:linauto_lr}
\widetilde\bfY = \argmin_{\rank{\bfY}\leq r} \  \bbE \norm[2]{\bfY X - X}^2 = \bbE \norm[2]{(\bfY-\bfI) X}^2 .
\end{equation}

We remark that the use of the 2-norm is a design choice and other metrics could be used to provide different designs.  For example, an alternative is to consider the Kulback-Leibner divergence between the two probability distributions of $X$ and $\bfY X$, where $\rank{\bfY}\leq r$. Another option would be to consider a Wasserstein distance \cite{piening2024paired}. Various design problems can be considered with different design criteria e.g., A-design, D-design \cite{atkinson2007optimum, pukelsheim2006optimal, ucinski2004optimal}.

Let us assume we are given the second moment of the random variable $X$, $$\bfGamma =  \bfL\bfL\t =  \bbE X X\t
$$
which is assumed to be symmetric and positive definite (SPD). Then the objective function in \Cref{eq:linauto_lr} can be written as
\begin{align*}\label{eq:bayesEXX} 
\bbE \norm[2]{(\bfY-\bfI) X}^2
&= \bbE \ \trace{X\t (\bfY-\bfI)\t (\bfY-\bfI) X } = \bbE \ \trace{(\bfY-\bfI)\t(\bfY-\bfI) X X\t}\\
&= \trace{(\bfY-\bfI)\t (\bfY-\bfI) \bfGamma }
= \norm[\fro]{ (\bfY-\bfI) \bfL }^2,
\end{align*}
where $\norm[\fro]{\mdot}$ denotes the Frobenius norm and the optimization problem reduces to
\begin{equation}
\label{eq:lr_opt}
\min_{\rank{\bfY}\leq r} \norm[\fro]{\bfY \bfL - \bfL }^2.
\end{equation}
We have the following result.

\begin{theorem}
\label{thm:fullrowrank}
Let matrix $\bfL \in \bbR^{n \times n}$ have full rank.  Additionally, let $\bfL = \bfU_{\bfL} \bfSigma_{\bfL} \bfV_{\bfL}\t$ be the SVD of $\bfL$, where $\bfSigma_{\bfL}$ is a diagonal matrix containing the nonzero singular values $\sigma_1 \geq \cdots \geq \sigma_n$, and the orthogonal matrices $\bfU_{\bfL} = [\bfu_1, \bfu_2, \ldots, \bfu_n]$ and $\bfV_{\bfL} = [\bfv_1, \bfv_2, \ldots, \bfv_n]$ contain the left and right singular vectors $\bfu_i, \bfv_j \in \bbR^n$ for $i, j = 1, \ldots, n$.  For positive integer $r\leq n$ we define the truncated SVD,
\begin{equation}
\label{eq: svd}
    \bfL_r = \bfU_{\bfL, r} \bfSigma_{\bfL, r} \bfV_{\bfL, r}\t = 
    \begin{bmatrix}
        \vertbar & \vertbar & & \vertbar \\
        \bfu_1 & \bfu_2 & \cdots & \bfu_r \\
        \vertbar & \vertbar& & \vertbar \\  
    \end{bmatrix}
    \begin{bmatrix}
        \sigma_1 & & & \\
         & \sigma_2 & & \\
         & & \ddots & \\
         & & & \sigma_r
    \end{bmatrix}
    \begin{bmatrix}
        \horzbar & \bfv_1\t & \horzbar \\
        \horzbar & \bfv_2\t & \horzbar \\
        & \vdots & \\
        \horzbar & \bfv_r\t & \horzbar \\
    \end{bmatrix}.
\end{equation}
Then
$$\widetilde \bfY = \bfU_{\bfL,r} \bfU_{\bfL,r}\t$$
is the solution to the minimization problem
$$\min_{\rank{\bfY}\leq r} \| \bfY \bfL - \bfL\|_\fro^2 ,$$
having minimal Frobenius norm $\norm[\fro]{\bfY}$. This solution is unique if and only if either $r\geq n$
or $1\leq r< n$ and $\sigma_r(\bfL)>\sigma_{r+1}(\bfL)$.
\end{theorem}

The proof follows directly from Theorem 3.1 in \cite{chung2017optimal}.
Notice that the theorem states that for $r<n$, the low-rank solution is unique if $\sigma_r(\bfL)>\sigma_{r+1}(\bfL)$, but the decomposition of $\widetilde \bfY$ into encoder $\widetilde \bfE$ and decoder $\widetilde \bfD$ is not unique since one could define for any invertible $r \times r$ matrix $\bfK,$
\begin{equation}\label{eq:aes}
    \widetilde \bfY = \underbrace{\bfU_{\bfL,r}  \bfK}_{\widetilde \bfD} \underbrace{\bfK^{-1}\bfU_{\bfL,r}\t}_{\widetilde \bfE}.
\end{equation}
The Bayes risk minimization approach requires the second moment of a random variable $X$, i.e., $\bfGamma$.  However, for problems where only samples or training data are available, we must consider an empirical Bayes risk minimization approach or a hybrid approach.

\subsubsection{An empirical Bayes risk minimization approach}
\label{subsub:empBayes linear}
Suppose we are given realizations $\bfx_1,\ldots, \bfx_N$ of $X$.  Let $\bfX = [\bfx_1,\ldots, \bfx_N] \in \bbR^{n \times N}$ and define the sample second moment matrix,
\begin{align}  \widebar \bfGamma & = \tfrac{1}{N-1}\sum_{j = 1}^N \bfx_j \bfx_j\t = \tfrac{1}{N-1} \bfX \bfX\t.
\end{align}
We describe two approaches to exploit the samples.
One approach would be to use $\widebar \bfGamma$ in the Bayes risk minimization approach described in \Cref{sub:ae_Bayes}. That is, the optimal linear autoencoder is given by
$$\widehat \bfY = \bfU_{\widebar \bfL,r} \bfU_{\widebar \bfL,r}\t, $$
where
$\widebar \bfGamma =\widebar  \bfL \widebar \bfL\t$ with potentially added SPD matrix to ensure that $\widebar \bfGamma$ is positive definite.

Another approach is to work directly with the samples, using them to approximate the Bayes risk minimization problem \cref{eq:linauto_lr} with an empirical Bayes risk minimization problem,
\begin{equation}
\min_{\rank{\bfY}\leq r} \tfrac{1}{N}\sum_{j = 1}^N \norm[2]{(\bfY-\bfI) \bfx_j}^2.
\end{equation}
Thus, obtaining an optimal linear autoencoder using an empirical Bayes risk minimization approach corresponds to solving 
\begin{equation}
    \label{eq:empBayesae}
    \widehat \bfY = \argmin_{\rank{\bfY}\leq r}||(\bfY - \bfI_n)\bfX||_\fro^2.
\end{equation}
If $\bfX$ has full row rank, \Cref{thm:fullrowrank} provides the optimal low-rank matrix.  However, this may not be the case, so we derive the results starting from Theorem 3.1 in \cite{chung2017optimal}.

Let $k = \rank{\bfX}$ and $\bfX = \bfU_{\bfX} \bfSigma_{\bfX} \bfV_{\bfX}\t$ be the SVD of $\bfX$, where $\bfSigma_{\bfX} \in \bbR^{n \times N}$ is a diagonal matrix containing the singular values of $\bfX$, and the orthogonal matrices $\bfU_{\bfX} = [\bfu_1, \bfu_2, \ldots, \bfu_n]\in\bbR^{n\times n}$ and $\bfV_{\bfX} = [\bfv_1, \bfv_2, \ldots, \bfv_N]\in\bbR^{N\times N}$ contain the left and right singular vectors $\bfu_i \in \bbR^n$ and $\bfv_j \in \bbR^N$ for $i = 1, \ldots, n$ and $j=1,\ldots,N$, respectively.  We will continue to use the truncated SVD notation introduced in \Cref{eq: svd}. 
From \cite{chung2017optimal}, for positive integer $r\leq k$, the minimizer $\widehat \bfY$ is given by
\begin{align*}
\widehat \bfY &= \left(\bfX \bfV_{\bfX,k} \bfV_{\bfX,k}\t\right)_r \bfV_{\bfX} \bfSigma_{\bfX}^{\dagger} \bfU_{\bfX}\t
= \left( \bfU_{\bfX} \bfSigma_{\bfX} 
    \begin{bmatrix}
        \bfI_k\\
        \bfzero
    \end{bmatrix}
    \bfV_{\bfX,k}\t \right)_r \bfV_{\bfX} \bfSigma_{\bfX}^{\dagger} \bfU_{\bfX}\t \\
&= \bfU_{\bfX,r} \bfSigma_{\bfX,r}\bfV_{\bfX,r}\t \bfV_{\bfX} \bfSigma_{\bfX}^{\dagger} \bfU_{\bfX}\t
= \bfU_{\bfX,r} \bfSigma_{\bfX,r} 
    \begin{bmatrix}
        \bfI_r & \bf0
    \end{bmatrix} \bfSigma_\bfX^\dagger \bfU_\bfX\t 
=\bfU_{\bfX,r} \bfU_{\bfX,r}\t.
\end{align*} 
Thus, one optimal choice of the linear encoder and decoder in an empirical Bayes risk sense is $\widehat\bfE=\bfU_{\bfX,r}\t$ and $\widehat\bfD=\bfU_{\bfX,r}$, respectively.  Again, the decomposition $\widehat \bfY = \widehat \bfD \widehat \bfE$ is not unique since 
\begin{equation}
\label{eq:empirical ae}
    \widehat\bfE=\bfK^{-1} \bfU_{\bfX,r}\t \quad \text{ and } \quad \widehat\bfD=\bfU_{\bfX,r} \bfK
\end{equation}
satisfies \Cref{eq:empBayesae} for any invertible $r\times r$ matrix $\bfK$.

Next, we discuss how both the Bayes risk and empirical Bayes risk interpretations of linear autoencoders can be integrated within the PAIR framework.

\subsection{Theoretical investigations for fully linear PAIR}\label{sub:theory_linmap}
In this section, we consider the special case where the autoencoder mappings $\Phi_{\rm ae}^\bfx$ and $\Phi_{\rm ae}^\bfb$ and the latent space mappings $\bfm$ and $\bfm^\dagger$ are all linear.
We describe how the resulting Bayes risk and empirical Bayes risk minimization problems lead to interesting numerical linear algebra problems, and we provide closed-form solutions to the low-rank approximation problems.

\subsubsection{Bayes risk minimization} 
\label{subsubsection: linear latent maps Bayes risk}
For random variable $X$ with finite first moment and SPD second moment $\bfGamma_\bfx=\bfL_\bfx\bfL_\bfx\t$, let us assume we have a linear autoencoder with encoder $\bfE_\bfx \in \bbR^{r_\bfx \times n}$ and decoder $\bfD_\bfx\in \bbR^{n \times r_\bfx}$ with $r_\bfx \leq n$.  This defines a latent variable $Z_\bfx=\bfE_\bfx X$. In the PAIR framework, this autoencoder for the target space is paired with another autoencoder for the input space, where an optimal mapping connects their latent spaces. 

Let random variable $B$ be related to $X$ by $B=\bfA X + \eps$, from \Cref{eq:inverseproblem} with linear forward model. Assume that the noise $\eps$ is independent of $X$ with mean $\bbE(\eps)=\bf0$ and SPD covariance $\bbE(\eps\eps\t)=\bfGamma_\eps=\bfL_\eps\bfL_\eps\t$, then $B$ has SPD second moment $\bfGamma_\bfb = \bfA \bfGamma_\bfx \bfA\t + \bfGamma_\eps$, and factorization $\bfGamma_\bfb=\bfL_\bfb \bfL_\bfb\t$.  Assume that we also have a linear autoencoder for $B$, with encoder $\bfE_\bfb\in\bbR^{r_\bfb \times q}$ and decoder $\bfD_\bfb\in\bbR^{q\times r_\bfq}$ with $r_\bfb \leq q$, and define the latent input variable $Z_\bfb=\bfE_\bfb B=\bfE_\bfb (\bfA X + \eps)$.  

\begin{theorem} \label{thm:linmap}
Let $X$ and $B$ be the random variables defined at the beginning of this section with linear autoencoders $\bfD_\bfx\bfE_\bfx$ and $\bfD_\bfb\bfE_\bfb$.  

If $\bfE_\bfx$ has full row-rank, then the optimal linear forward mapping between latent spaces with minimal norm is given by
\begin{equation}\label{eq:linmap}
    \widetilde\bfM = \bfE_\bfb \bfA \bfGamma_\bfx \bfE_\bfx\t \left(\bfE_\bfx \bfGamma_\bfx \bfE_\bfx\t\right)^{-1} \in \argmin_\bfM \ \bbE \norm[2]{\bfM Z_\bfx - \bfE_\bfb \bfA X}^2.
    \end{equation}
    
If $\bfE_\bfb$ has full row-rank, then the optimal linear inverse map between latent spaces with minimal norm is given by
\begin{equation}\label{eq:linmapInv}
    \widetilde\bfM^\dagger = \bfE_\bfx \bfGamma_\bfx\t \bfA\t \bfE_\bfb\t \left(\bfE_\bfb \bfGamma_\bfb \bfE_\bfb\t \right)^{-1} \in \argmin_{\bfM^\dagger} \ \bbE \norm[2]{\bfM^\dagger Z_\bfb - Z_\bfx }^2.
\end{equation}
\end{theorem}

\begin{proof}
For $\widetilde\bfM$, we can use properties of the two-norm, trace, and expectation to rewrite the objective function in \Cref{eq:linmap} as 
$\norm[\fro]{(\bfM\bfE_\bfx-\bfE_\bfb\bfA)\bfL_\bfx}^2$,
following similar steps used to reformulate \Cref{eq:linauto_lr} as (\ref{eq:lr_opt}). 
The solution to this least squares problem with minimal norm is $\widetilde\bfM = \bfE_\bfb\bfA\bfL_\bfx(\bfE_\bfx\bfL_\bfx)^\dagger$. Note that $\bfE_\bfx$ has full row rank, so by the definition of the Moore-Penrose pseudo-inverse,
\begin{align*}
    \widetilde \bfM &= \bfE_\bfb\bfA\bfL_\bfx\bfL_\bfx\t\bfE_\bfx\t(\bfE_\bfx\bfL_\bfx\bfL_\bfx\t\bfE_\bfx\t)^{-1}
    = \bfE_\bfb\bfA\bfGamma_\bfx\bfE_\bfx\t(\bfE_\bfx\bfGamma_\bfx\bfE_\bfx\t)^{-1}.
\end{align*}

For $\widetilde\bfM^\dagger$, we can use properties of the two-norm, expectation, and trace to rewrite the objective function in \Cref{eq:linmapInv} as
\begin{align*}
    \trace{\bfM^\dagger \bfE_\bfb \bbE(BB\t) \bfE_\bfb\t (\bfM^\dagger)\t}
    - 2 \trace{\bfM^\dagger \bfE_\bfb \bbE(BX\t) \bfE_\bfx\t} + \trace{\bfE_\bfx \bbE(XX\t) \bfE_\bfx\t}.
\end{align*}
Note that $X$ and $\eps$ are independent, with $\bbE(\eps)=\bf0$ and finite first moment of $X$, so 
$$\bbE(BX\t) = \bbE((\bfA X+\eps)X\t) = \bfA \bbE(XX\t) + \bbE(\eps)\bbE(X\t) = \bfA \bfGamma_\bfx.$$
Substituting this, and the other second moments defined in the beginning of this subsection, the objective function becomes
\begin{align*}
\norm[\fro]{\bfM^\dagger \bfE_\bfb \bfA \bfL_\bfx - \bfE_\bfx \bfL_\bfx}^2 + \norm[\fro]{\bfM^\dagger \bfE_\bfb \bfL_\eps}^2 = \norm[\fro]{\bfM^\dagger \begin{bmatrix}
    \bfE_\bfb \bfA \bfL_\bfx & \bfE_\bfb \bfL_\eps
\end{bmatrix} - \begin{bmatrix}
    \bfE_\bfx\bfL_\bfx & \bfzero
\end{bmatrix}}^2.
\end{align*}
The optimal solution to this least squares problem with minimal norm is given by $\bfM^\dagger = \begin{bmatrix}\bfE_\bfx\bfL_\bfx&\bf0\end{bmatrix}\begin{bmatrix}\bfE_\bfb\bfA\bfL_\bfx & \bfE_\bfb \bfL_\eps\end{bmatrix}^\dagger.$  Note that even when $\bfA$ is rank-deficient, $\begin{bmatrix}\bfE_\bfb\bfA\bfL_\bfx & \bfE_\bfb \bfL_\eps\end{bmatrix}$ has full row rank because $\bfE_\bfb\bfL_\eps$ is necessarily full rank, and thus we can further simplify using the definition of the Moore-Penrose pseudo-inverse,
\begin{align*}
    \widetilde\bfM^\dagger &= \begin{bmatrix}\bfE_\bfx\bfL_\bfx&\bf0\end{bmatrix}\begin{bmatrix}\bfL_\bfx\t \bfA\t \bfE_\bfb\t \\ \bfL_\eps\t \bfE_\bfb\t\end{bmatrix}\left(\begin{bmatrix}\bfE_\bfb\bfA\bfL_\bfx & \bfE_\bfb \bfL_\eps\end{bmatrix}\begin{bmatrix}\bfL_\bfx\t \bfA\t \bfE_\bfb\t \\ \bfL_\eps\t \bfE_\bfb\t\end{bmatrix}\right)^{-1} \\
    &=\bfE_\bfx \bfGamma_\bfx \bfA\t \bfE_\bfb\t \left(\bfE_\bfb \bfA \bfGamma_\bfx\t \bfA\t \bfE_\bfb\t + \bfE_\bfb \bfGamma_\eps\t \bfE_\bfb\t \right)^{-1}
    =\bfE_\bfx \bfGamma_\bfx \bfA\t \bfE_\bfb\t \left(\bfE_\bfb \bfGamma_\bfb \bfE_\bfb\t \right)^{-1}.
\end{align*}

\end{proof}

We next describe the full PAIR framework with optimal linear autoencoders and optimal linear latent maps.  For optimal linear autoencoders for $X$ and $B$, we have encoders $\widetilde\bfE_\bfx,\widetilde\bfE_\bfb$ and decoders $\widetilde\bfD_\bfx,\widetilde\bfD_\bfb$ according to \Cref{eq:aes}, i.e.,
\begin{equation}
\begin{aligned}\label{eq:optencodedecode}
    &\widetilde \bfE_\bfx = \bfK_\bfx^{-1} \bfU_{\bfL_\bfx,r_\bfx}^\top,
    &\widetilde \bfE_\bfb = \bfK_\bfb^{-1} \bfU_{\bfL_\bfb,r_\bfb}^\top,\\
    &\widetilde \bfD_\bfx = \bfU_{\bfL_\bfx,r_\bfx} \bfK_\bfx,
    &\widetilde \bfD_\bfb = \bfU_{\bfL_\bfb,r_\bfb} \bfK_\bfb,
\end{aligned}
\end{equation}
where the first subscript denotes which matrix an SVD factor belongs to and the second denotes where it is truncated (e.g.,  $\bfU_{\bfL_\bfx, r_\bfx}$ denotes the first $r_\bfx$ left singular vectors of $\bfL_\bfx$). We provide the following for the fully linear PAIR surrogates.
\begin{theorem}
\label{thm:pair}
Consider the linear autoencoders given in \Cref{eq:optencodedecode} and the linear mappings between latent spaces given in \Cref{thm:linmap} (Equations \ref{eq:linmap} and \ref{eq:linmapInv}), that are optimal in the Bayes risk minimization sense.  Let the linear PAIR forward surrogate be defined as $\widetilde\bfP = \widetilde\bfD_\bfb \widetilde \bfM \widetilde\bfE_\bfx$. Then
    \begin{equation}
    \label{eq:PAIR_forward}
    \widetilde\bfP = \bfU_{\bfL_\bfb, r_\bfb}\bfU_{\bfL_\bfb, r_\bfb}\t \bfA \bfU_{\bfL_\bfx, r_\bfx}\bfU_{\bfL_\bfx, r_\bfx}\t.\end{equation}  
Let the linear PAIR inverse surrogate be defined as $\widetilde\bfP^{\dagger} = \widetilde\bfD_\bfx \widetilde\bfM^\dagger \widetilde\bfE_\bfb$. Then
    \begin{equation}
    \label{eq:PAIR-inverse}
     \widetilde\bfP^{\dagger} = {\bfU_{\bfL_\bfx, r_\bfx}}\bfSigma_{\bfL_\bfx, r_\bfx}^2 \bfU_{\bfL_{r_\bfx}}\t \bfA\t {\bfU_{\bfL_\bfb, r_\bfb}} \bfSigma_{\bfL_\bfb, r_\bfb}^{-2} {\bfU_{\bfL_\bfb, r_\bfb}\t}.
    \end{equation}
\end{theorem}
Before proving this result, we make a few remarks. First, note that although each component of $\widetilde\bfP$ and $\widetilde\bfP^\dagger$ is optimal in a Bayes risk minimization sense, $\widetilde\bfP$ and $\widetilde\bfP^\dagger$ themselves do not necessarily minimize 
$\bbE\|\bfP X - B\|_2^2$ and $\bbE\|\bfP^\dagger B - X\|_2^2,$
respectively.
Second, we can interpret the mappings in \Cref{thm:pair} as fully data-driven, since within \Cref{eq:PAIR_forward,eq:PAIR-inverse} the operator $\bfA$ remains untouched, i.e., no rank constraint is imposed directly on $\bfA$. The matrices $\bfU_{\bfL_\bfx, r_\bfx}$ and  $\bfU_{\bfL_\bfb, r_\bfb}$, on the other hand, provide projections from and onto relevant lower-dimensional subspaces defined through the distributions of the random variables $X$ and $B$.  

\begin{proof}[Proof of \Cref{thm:pair}]
Using \Cref{eq:linmap}, through substitution of our optimal encoder and decoder choices and the SVD $\bfGamma_\bfx=\bfU_{\bfL_\bfx} \bfSigma_{\bfL_\bfx}^2 \bfU_{\bfL_\bfx}\t$, we find $\widetilde\bfP$ is equivalently
\begin{align*}
&\bfU_{\bfL_\bfb, r_\bfb}\bfK_\bfb
\bfK_\bfb^{-1} \bfU_{\bfL_\bfb,r_\bfb}\t \bfA \bfGamma_\bfx \bfU_{\bfL_\bfx, r_\bfx} \bfK_\bfx^{-\top}\left( \bfK_\bfx^{-1} \bfU_{\bfL_\bfx,r_\bfx}^\top \bfGamma_\bfx  \bfU_{\bfL_\bfx,r_\bfx}\bfK_\bfx^{-\top}\right)^{-1}
\bfK_\bfx^{-1} \bfU_{\bfL_\bfx, r_\bfx}\t\\
&= \bfU_{\bfL_\bfb, r_\bfb}\bfU_{\bfL_\bfb,r_\bfb}\t \bfA \bfU_{\bfL_\bfx}\bfSigma_{\bfL_\bfx}^2 \begin{bmatrix}\bfI_{r_\bfx}\\\bf0\end{bmatrix}  \left(\begin{bmatrix}\bfI_{r_\bfx}&\bf0\end{bmatrix} \bfSigma_{\bfL_\bfx}^2\begin{bmatrix}\bfI_{r_\bfx}\\\bf0\end{bmatrix} \right)^{-1} \bfU_{\bfL_\bfx, r_\bfx}\t\\
&= \bfU_{\bfL_\bfb, r_\bfb} \bfU_{\bfL_\bfb,r_\bfb}\t \bfA \bfU_{\bfL_\bfx, r_\bfx} \bfSigma_{\bfL_\bfx,r_\bfx}^2 \bfSigma_{\bfL_\bfx,r_\bfx}^{-2} \bfU_{\bfL_\bfx, r_\bfx}\t 
= \bfU_{\bfL_\bfb, r_\bfb} \bfU_{\bfL_\bfb,r_\bfb}\t \bfA \bfU_{\bfL_\bfx, r_\bfx} \bfU_{\bfL_\bfx, r_\bfx}\t.
\end{align*}
Using \Cref{eq:linmapInv}, through substitution of our encoder and decoder choices, we find
\begin{align*}
    \widetilde\bfP^\dagger
        &= {\bfU_{\bfL_\bfx, r_\bfx}} {\bfU_{\bfL_\bfx, r_\bfx}\t} \bfGamma_\bfx \bfA\t {\bfU_{\bfL_\bfb, r_\bfb}} {\left(\bfU_{\bfL_\bfb, r_\bfb}\t \bfGamma_\bfb \bfU_{\bfL_\bfb, r_\bfb}\right)^{-1}} {\bfU_{\bfL_\bfb, r_\bfb}\t}.
    \end{align*}
    Next, we substitute the SVDs of $\bfGamma_\bfx$ and $\bfGamma_\bfb$, and obtain
    \begin{align*}
        {\left(\bfU_{\bfL_\bfb, r_\bfb}\t \bfGamma_\bfb \bfU_{\bfL_\bfb, r_\bfb}\right)^{-1}} &= {\left(\bfU_{\bfL_\bfb, r_\bfb}\t \bfU_{\bfL_\bfb} \bfSigma_{\bfL_\bfb}^2 \bfU_{\bfL_\bfb}\t \bfU_{\bfL_\bfb, r_\bfb}\right)^{-1}}
        = \bfSigma_{\bfL_\bfb, r_\bfb}^{-2},\\
    \widetilde\bfP^\dagger
         &= {\bfU_{\bfL_\bfx, r_\bfx}}\bfSigma_{\bfL_\bfx, r_\bfx}^2 \bfU_{\bfL_\bfx, {r_\bfx}}\t \bfA\t {\bfU_{\bfL_\bfb, r_\bfb}} \bfSigma_{\bfL_\bfb, r_\bfb}^{-2} {\bfU_{\bfL_\bfb, r_\bfb}\t}.
    \end{align*}
\end{proof}

One feature of the PAIR method defined in the Bayes risk sense is that, without compression, it exactly recovers the forward map $\bfA$.  That is, when $r_\bfx=n$ and $r_\bfb=q$, \Cref{eq:PAIR_forward} yields $$\widetilde\bfP = \bfU_{\bfL_\bfb} \bfU_{\bfL_\bfb}\t \bfA \bfU_{\bfL_\bfx } \bfU_{\bfL_\bfx}\t = \bfA,$$
because both $\bfU_{\bfL_\bfx}$ and $\bfU_{\bfL_\bfb}$ are orthogonal. When $\bfA$ is invertible and $B$ is noiseless, this follows analogously with $\widetilde\bfP^\dagger = \bfA^{-1}$.  In general, when $B$ is defined with noise, we do not recover $\bfA^{-1}$, even when $\bfA$ is invertible.  From \Cref{eq:PAIR-inverse},  when $r_\bfx=n$, and $r_\bfb=q$, 
$$\widetilde\bfP^\dagger = {\bfU_{\bfL_\bfx}}\bfSigma_{\bfL_\bfx}^2 \bfU_{\bfL_\bfx}\t \bfA\t {\bfU_{\bfL_\bfb}} \bfSigma_{\bfL_\bfb}^{-2} {\bfU_{\bfL_\bfb}\t} = \bfGamma_\bfx \bfA\t \bfGamma_\bfb^{-1}.$$

Next, we turn to the case where we are given samples or training data.
Using an empirical Bayes risk minimization approach, we obtain approximations of the optimal linear latent space mappings and PAIR surrogate forward and inverse approximations.

\subsubsection{Empirical Bayes risk minimization}
\label{sub:empBayes latent maps}
When working with samples, one approach to construct linear latent maps is to calculate sample second moment matrices $\widebar\bfGamma_\bfb=\widebar\bfL_\bfb\widebar\bfL_\bfb\t$ and $\widebar\bfGamma_\bfx=\widebar\bfL_\bfx\widebar\bfL_\bfx\t$ and to use these approximations to form $\bfM$ and $\bfM^\dagger$ following the definitions in \Cref{eq:linmap,eq:linmap}.  However, this approach has some disadvantages; to define $\bfM$, the full-scale forward map $\bfA$ is assumed to be linear and to define $\bfM^\dagger$ requires linear $\bfA$ with access to $\bfA\t$.  Another disadvantage is that the sample second moment matrix is typically hard to approximate in higher dimensions and may not be SPD with only a small number of samples.

Thus, we consider an alternative approach to form mappings that does not require an explicit or linear forward mapping $\bfA$.  Let us assume we are given realizations $\bfx_1,\ldots, \bfx_{N} \in \bbR^{n}$ of random variable $X$, and realizations $\bfb_1,\ldots, \bfb_{Q} \in \bbR^{q}$ of random variable $B$.  Let $\bfX = [\bfx_1,\ldots, \bfx_N] \in \bbR^{n \times N}$ and $\bfB = [\bfb_1,\ldots, \bfb_{Q}] \in \bbR^{q \times Q}$. 

Consider a PAIR network with linear autoencoders, with encoders $\bfE_\bfx \in \bbR^{r_\bfx \times n}$ and $\bfE_\bfb\in \bbR^{r_\bfb \times q}$ and decoders $\bfD_\bfx\in \bbR^{n \times r_\bfx}$ and $\bfD_\bfb\in\bbR^{n \times r_\bfb}$ with $r_\bfb\leq q$, $r_\bfx\leq n$.

Matrices $\bfZ_\bfx = \bfE_\bfx \bfX$ and $\bfZ_\bfb = \bfE_\bfb \bfB$ contain latent representations of the realizations (for $\bfX$ and $\bfB$ respectively).
To find optimal linear mappings $\widehat\bfM$ and $\widehat\bfM^\dagger$ between latent spaces, we consider the following optimization problems,
\begin{align*}
      \widehat\bfM \in \argmin_{\bfM} \ \norm[\fro]{\bfM \bfZ_\bfx - \bfZ_\bfb} \quad \text{ and } \quad \widehat\bfM^\dagger \in \argmin_{\bfM^\dagger} \ \norm[\fro]{\bfM^\dagger \bfZ_\bfb - \bfZ_\bfx}.
\end{align*}
The solution with minimal norm is given by the Moore-Penrose pseudo-inverse, 
\begin{equation}\label{eq:empirical m and mdagger}
    \widehat\bfM = \bfZ_\bfb \bfZ_\bfx^\dagger \quad \text{ and } \quad \widehat\bfM^\dagger = \bfZ_\bfx \bfZ_\bfb^\dagger.
\end{equation} 
Assuming $\bfZ_\bfx$ and $\bfZ_\bfb$ have full row rank, equivalently, $\widehat\bfM = \bfZ_\bfb \bfZ_\bfx\t(\bfZ_\bfx\bfZ_\bfx\t)^{-1}$ and $\widehat\bfM^\dagger = \bfZ_\bfx \bfZ_\bfb\t(\bfZ_\bfb \bfZ_\bfb\t)^{-1}$.

\begin{proposition}
In the empirical Bayes risk minimization sense, if optimal linear autoencoders with $\widehat\bfE_\bfx$, $\widehat\bfE_\bfb$, $\widehat\bfD_\bfx$, and $\widehat\bfD_\bfb$ as defined by \Cref{eq:empirical ae} are used with $\sigma_{r_\bfx}(\bfX)>0$, then the PAIR forward surrogate $\widehat\bfP=\widehat\bfD_\bfb \widehat\bfM \widehat\bfE_\bfx$ is given by 

\begin{equation}
\label{eq:PAIRempbayes}
    \widehat\bfP =\bfU_{\bfB,r_\bfb} \bfSigma_{\bfB,r_\bfb} \bfV_{\bfB,r_\bfb}\t \bfV_{\bfX,r_\bfx} \bfSigma_{\bfX,r_\bfx}^{-1} \bfU_{\bfX,r_\bfx}\t= \bfB_{r_\bfb} \bfX_{r_\bfx}^\dagger.
\end{equation}
\end{proposition}

\begin{proof} From \Cref{eq:empirical m and mdagger}, we have 
\begin{align*}
\widehat\bfP &= \widehat\bfD_\bfb \bfZ_\bfb \bfZ_\bfx\t(\bfZ_\bfx\bfZ_\bfx\t)^{-1} \widehat\bfE_\bfx =\widehat\bfD_\bfb \widehat\bfE_\bfb\bfB\bfX\t\widehat\bfE_\bfx\t\left(\widehat\bfE_\bfx\bfX\bfX\t\widehat\bfE_\bfx\t\right)^{-1} \widehat\bfE_\bfx.
\end{align*}
Using \Cref{eq:empirical ae} to define the optimal encoders  and the SVD of $\bfX$, we find
\begin{align*}    \left(\widehat\bfE_\bfx\bfX\bfX\t\widehat\bfE_\bfx\t\right)^{-1}
    &= \left(\bfK_\bfx^{-1}\begin{bmatrix}\bfI_{r_\bfx}&\bf0\end{bmatrix}\bfSigma_\bfX\bfSigma_\bfX\t\begin{bmatrix}\bfI_{r_\bfx}\\\bf0\end{bmatrix}\bfK_\bfx^{-\top}\right)^{-1}
    = \bfK_\bfx^{\top} \bfSigma_{\bfX,r_\bfx}^{-2} \bfK_\bfx\\
\widehat\bfP &=\bfU_{\bfB,r_\bfb} \bfK_\bfb \bfK_\bfb^{-1} \bfU_{\bfB,r_\bfb}\t \bfB \bfX\t \bfU_{\bfX,r_\bfx} \bfK_\bfx^{-\top} \bfK_\bfx^{\top} \bfSigma_{\bfX,r_\bfx}^{-2} \bfK_\bfx \bfK_\bfx^{-1} \bfU_{\bfX,r_\bfx}\t\\
&= \bfU_{\bfB,r_\bfb}\begin{bmatrix}\bfI_{r_\bfb}&\bf0\end{bmatrix}\bfSigma_\bfB\bfV_\bfB\t \bfV_\bfX\bfSigma_\bfX\t\begin{bmatrix}\bfI_{r_\bfx}\\\bf0\end{bmatrix}\bfSigma_{\bfX,r_\bfx}^{-2} \bfU_{\bfX,r_\bfx}\t
=\bfB_{r_\bfb} \bfX_{r_\bfx}^\dagger.
\end{align*}

\end{proof}
With $\sigma_{r_\bfb}(\bfB)>0$, the PAIR inverse surrogate $\widehat\bfP=\widehat\bfD_\bfx\widehat\bfM^\dagger\widehat\bfE_\bfb=\bfX_{r_\bfx}\bfB_{r_\bfb}^\dagger$ follows analogously.

When $\bfB=\bfA\bfX$ (linear, noiseless), and we consider the PAIR forward surrogate without compression ($r_\bfb=q$ and $r_\bfx=n$) defined with optimal empirical Bayes risk choices, we make a note of the following connections to classic linear algebra problems from \Cref{eq:PAIRempbayes}:
\begin{itemize}
    \item If $\bfX$ has full row rank, then 
    $$\widehat\bfP = \bfB \bfX^\dagger = \bfA \bfX \bfX\t (\bfX\bfX\t)^{-1} = \bfA.$$
    \item If $\rank\bfX=k<n$, then 
\begin{align*}
\widehat\bfP = \bfA \bfU_\bfX \begin{bmatrix}\bfSigma_{\bfX,k}&\bf0\\\bf0&\bf0\end{bmatrix} \begin{bmatrix}\bfSigma_{\bfX,k}^{-1}&\bf0\\\bf0&\bf0\end{bmatrix} \bfU_\bfX\t =\bfA \begin{bmatrix}\bfU_{\bfX,k} \bfU_{\bfX,k}\t & \bf0\\\bf0&\bf0\end{bmatrix}.
\end{align*}
\end{itemize}
That is, when $\bfX$ is full rank, we exactly recover $\widehat\bfP=\bfA$, but when $\bfX$ is not full rank, we can only exactly recover $\widehat\bfP \bfX=\bfB$. 

Moreover, when $\bfB=\bfA\bfX$ and $\bfA$ is invertible, and we consider the PAIR inverse surrogate without compression ($r_\bfb=q$ and $r_\bfx=n$) defined with optimal empirical Bayes risk choices:
\begin{itemize}
    \item If $\bfX$ has full row rank, then $$\widehat\bfP^\dagger = \bfX \bfB^\dagger
        = \bfX\bfX\t\bfA\t\left(\bfA\bfX\bfX\t\bfA\t\right)^{-1} 
        = \bfA^{-1}.$$
    \item If $\rank\bfX= k < n$, then $$\widehat\bfP^\dagger = \bfA^{-1} \bfB \bfB^\dagger =\bfA^{-1} \begin{bmatrix}\bfU_{\bfB,k} \bfU_{\bfB,k}\t & \bf0\\\bf0&\bf0\end{bmatrix}.$$
\end{itemize}

Intuitively, these results show that the PAIR forward and inverse surrogates provide approximations that are data-specific, in that the autoencoders provide compression using data-informed projections and the mapping between latent spaces is learned from data. Thus, we would expect that another sample $\bfx$ that is within the distribution of the samples used to create the PAIR network would be mapped well with $\widehat\bfP$.  More specifically, if $\bfx$ is in the column space of $\bfU_{\bfX,k}$, then $\bfA\bfx=\widehat\bfP \bfx$.  We only expect $\widehat\bfP=\bfA$ when the column space of $\bfU_\bfX$ is $\bbR^n$.  Otherwise, $\widehat\bfP$ approximates a matrix with the same action as $\bfA$ for vectors in the span of what we used to construct the PAIR network.

\section{Numerical results}
\label{sec:numerics}
In this section, we present numerical results and illustrations of the theory provided above.  

It is common in applications to work only with data pairs, $\left\{(\bfb_j,\bfx_j)\right\}_{j = 1}^J$ (related by \Cref{eq:inverseproblem}).  However, obtaining good surrogates can be challenging if the number of paired samples is small.  One of the benefits of the PAIR framework is that learning the autoencoders and mapping between latent spaces may be considered independent processes.  Thus, each autoencoder, $\Phi_{\rm ae}^\bfx$ and $\Phi_{\rm ae}^\bfb$, can be constructed in parallel, as a self-supervised learning task and different sizes and types of the datasets can be used. That is if unpaired samples of the data and/or parameters are available (i.e., as $\{\bfx_j\}_{j=1}^{N}$ and $\{\bfb_j\}_{j=1}^{Q}$) they can be added and used to construct both $\Phi_{\rm ae}^\bfx$ and $\Phi_{\rm ae}^\bfx$.  

In \Cref{sub:SheppLogan}, we describe a PAIR network with linear autoencoders for an application in computed tomography (CT) imaging, illustrating the theory described in \Cref{sub:theory_autoencoder}. In particular, we use the empirical Bayes risk minimization interpretation and define
\begin{align*}
    e_\bfx(\bfx) &= \widehat\bfE_\bfx \bfx= \bfU_{\bfX,r_\bfx}^\top\bfx,
    & e_\bfb(\bfb) &= \widehat\bfE_\bfb \bfb= \bfU_{\bfB,r_\bfb}^\top\bfb,\\
    d_\bfx(\bfx) &= \widehat\bfD_\bfx \bfx= \bfU_{\bfX,r_\bfx}\bfx,
    &d_\bfb(\bfb) &= \widehat\bfD_\bfb \bfb= \bfU_{\bfB,r_\bfb}\bfb,
\end{align*}
where $\bfX = [\bfx_1,\ldots, \bfx_J] \in \bbR^{n \times J}$, $\bfB = [\bfb_1,\ldots, \bfb_J] \in \bbR^{q \times J}$, and $r_\bfx$ and $r_\bfb$ are the dimensions of the latent spaces of $\Phi_{\rm ae}^\bfx$ and $\Phi_{\rm ae}^\bfb$, respectively.

Then, in \Cref{sub:MNIST}, we discuss a PAIR network with nonlinear convolutional neural networks (CNNs) for an application in image deblurring. Here, $\Phi_{\rm ae}^\bfx$ and $\Phi_{\rm ae}^\bfb$ are parameterized by $\bftheta^{\rm e}_\bfx,$ $\bftheta^{\rm d}_\bfx,$ $\bftheta^{\rm e}_\bfb,$ and $\bftheta^{\rm d}_\bfb$ (superscripts denote belonging to an encoder or decoder, while subscripts denote belonging to the $\bfx$ or $\bfb$ autoencoders).  To learn the parameters of the autoencoders, we solve the following optimization problems,
\begin{equation*} 
    \min_{\bftheta^{\rm e}_\bfx, \bftheta^{\rm d}_\bfx} 
    \tfrac{1}{J}\sum_{j=1}^J\left|\left|d_\bfx (e_\bfx(\bfx_j;\bftheta^{\rm e}_\bfx); \bftheta^{\rm d}_\bfx) - \bfx_j\right|\right|_2^2
    \hspace{1ex}    \text{and}    \hspace{1ex}
    \min_{\bftheta^{\rm e}_\bfb, \bftheta^{\rm d}_\bfb} \tfrac{1}{J}\sum_{j=1}^J\left|\left|d_\bfb (e_\bfb(\bfb_j;\bftheta^{\rm e}_\bfb); \bftheta^{\rm d}_\bfb) - \bfb_j\right|\right|_2^2.
\end{equation*}
Similar to the linear case, the autoencoders
define the latent spaces $\calZ_\bfx$ and $\calZ_\bfb$ with compressed representations of $\bfx$ and $\bfb$. These representations are given by \begin{align*}
\bfZ_\bfx = 
    \begin{bmatrix}
    | &  & |\\
    e_\bfx(\bfx_1) & \cdots & e_\bfx(\bfx_N)\\
    | &  & |
\end{bmatrix}
\quad
\text{and}
\quad
\bfZ_\bfb = 
\begin{bmatrix}
    | &  & |\\
    e_\bfb(\bfb_1) & \cdots & e_\bfb(\bfb_N)\\
    | &  & |
\end{bmatrix}.
\end{align*}

For both the linear and nonlinear autoencoder examples, we use linear mappings between the latent spaces to connect the autoencoders, given by
\begin{align*}
     \quad \widehat\bfM = \bfZ_\bfb \bfZ_\bfx^\dagger \quad \text{ and } \quad\widehat\bfM^\dagger = \bfZ_\bfx \bfZ_\bfb^\dagger,
\end{align*}
as discussed in \Cref{sub:empBayes latent maps} (\Cref{eq:empirical m and mdagger}).  This task is supervised, requiring $\{(\bfb_j,\bfx_j)\}_{j=1}^J$ pairs.  Code will be provided at \href{https://github.com/emmahart2000/PAIR}{\tt github.com/emmahart2000/PAIR}.

\subsection{Linear PAIR for computed tomography}\label{sub:SheppLogan}
In CT reconstruction, the objective is to obtain images that contain information regarding the internal structure or anatomy of an object, based on projections or measurements collected from the object's exterior. In the following, we use a dataset consisting of phantom images $\bfx$ and sinograms $\bfb$, where the forward model $\bfA$ relates the sinogram and phantom by $\bfA\bfx+ \bfeps=\bfb$.  The forward model corresponds to rotating a source that sends radiation through the object to a detector located on the opposite side of the object.

Each target image $\bfx \in \mathbb{R}^{64\cdot 64}$ is a randomized Shepp-Logan phantom, representing a brain \cite{randomSheppLogan}.  Each sinogram $\bfb\in\bbR^{90\cdot 36}$ is made from $\bfx$ by simulating the forward model of an X-ray CT and adding $5\%$ white noise \cite{gazzola2019ir}.  We randomly generate 10,000 phantoms and simulate the forward process to create 10,000 corresponding sinograms, adding white noise and using these for the self-supervised task of constructing the sinogram autoencoder (excluding the corresponding phantoms from any other use).  We generate another 8,000 phantoms for the self-supervised task of constructing the phantom autoencoder, then we generate another 8,000 phantoms and simulate the forward process to create 8,000 phantom-sinogram pairs that are used for the supervised task of constructing latent mappings.  An additional 2,000 phantom-sinogram pairs were generated for testing and comparing methods.

To construct the autoencoders, we begin by vectorizing the input and target images and storing them in columns of $\bfB \in \mathbb{R}^{3,240\times 10,000}$ and $\bfX\in \mathbb{R}^{4,096\times 8,000}$, respectively.  This PAIR network is fully linear, so the encoder and decoder for $\bfx$ with latent dimension $r_\bfx$ is given by $\widehat\bfE_\bfx = \bfU_{\bfX,r_\bfx}\t$ and $\widehat\bfD_\bfx = \bfU_{\bfX,r_\bfx}$, following \Cref{subsub:empBayes linear} (\Cref{eq:empirical ae} with $\bfK=\bfI$).  Analogously, for $\bfb$ and latent dimension $r_\bfb$, $\widehat\bfE_\bfb = \bfU_{\bfB,r_\bfb}\t$ and $\widehat\bfD_\bfb = \bfU_{\bfB,r_\bfb}$.

As we increase the dimension of the latent space, we expect to attain more accurate autoencoders, and hence more accurate inversion and forward propagation through the PAIR network. In \Cref{fig:linearPAIRex}, we provide relative error norms, averaged over 2,000 testing images, for various ranks.  Note that both the input and target autoencoders perform better for larger ranks, as expected. We observe that as the rank increases, the forward PAIR surrogate provides better approximations of the forward model.  For a fixed rank $r$, the PAIR approximation is even better than the traditional truncated SVD (TSVD) approximation $\bfA_r$, thus demonstrating the benefits of leveraging data for improved surrogate modeling.  For the inverse mapping, we first note that due to ill-posedness of the problem, TSVD reconstructions exhibit a well-known phenomenon called semiconvergence, whereby the reconstruction errors initially decrease as the rank is increased, but as smaller singular values are included in the reconstruction, the reconstruction errors increase \cite{hansen2010discrete}.  We expect and observe similar behavior for the PAIR inverse surrogate. However, compared to TSVD reconstructions, the PAIR inverse surrogate reconstructions achieve smaller average reconstruction errors for all ranks, with the smallest reconstruction errors attained around rank 2,600.

\begin{figure}[tb]
    \centering
    \includegraphics[width=0.75\linewidth]{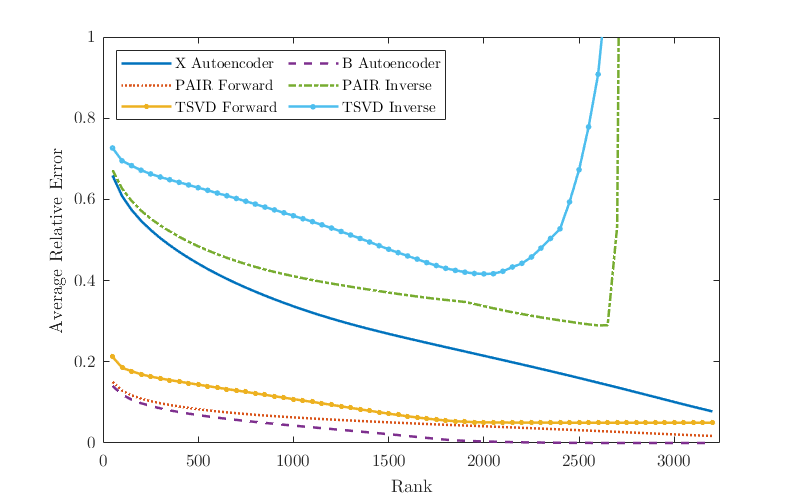}
    \caption{Linear PAIR results for CT. Relative error is averaged over the 2,000 testing images according to the following: X autoencoder, $||\widehat\bfD_\bfx\widehat\bfE_\bfx \bfx - \bfx||_2/||\bfx||_2$; B autoencoder, $||\widehat\bfD_\bfb \widehat\bfE_\bfb \bfb - \bfb||_2/||\bfb||_2$;
    PAIR Forward, $||\widehat\bfP \bfx - \bfb||_2/||\bfb||_2$;
    PAIR Inverse, $||\widehat\bfP^\dagger \bfb - \bfx ||_2/||\bfx||_2$;
    TSVD Forward, $||\bfU_{\bfA,r} \bfSigma_{\bfA,r} \bfV^\top_{\bfA,r}\bfx - \bfb||_2/||\bfb||_2$; and
    TSVD Inverse, $||\bfV_{\bfA,r} \bfSigma_{\bfA,r}^{-\top} \bfU_{\bfA,r}^\top \bfb- \bfx||_2/||\bfx||_2$.
    }
    \label{fig:linearPAIRex}
\end{figure}

To provide better intuition about the compression and reconstruction process, we provide in \Cref{fig:latentexamples} reconstructions corresponding to one of the testing images.  The true phantom and observed sinogram are provided in the top right corner.  The top two rows of images contain image representations for autoencoders of different ranks for $\bfb$ and $\bfx$, respectively. The bottom two rows of images contain reconstructions for both the PAIR forward surrogate and the PAIR inverse surrogate for different sizes of the latent space.  Notice that the quality of the surrogates appears to be limited by the quality of the autoencoders.

\begin{figure}[tb]
    \centering
    \includegraphics[width=\linewidth]{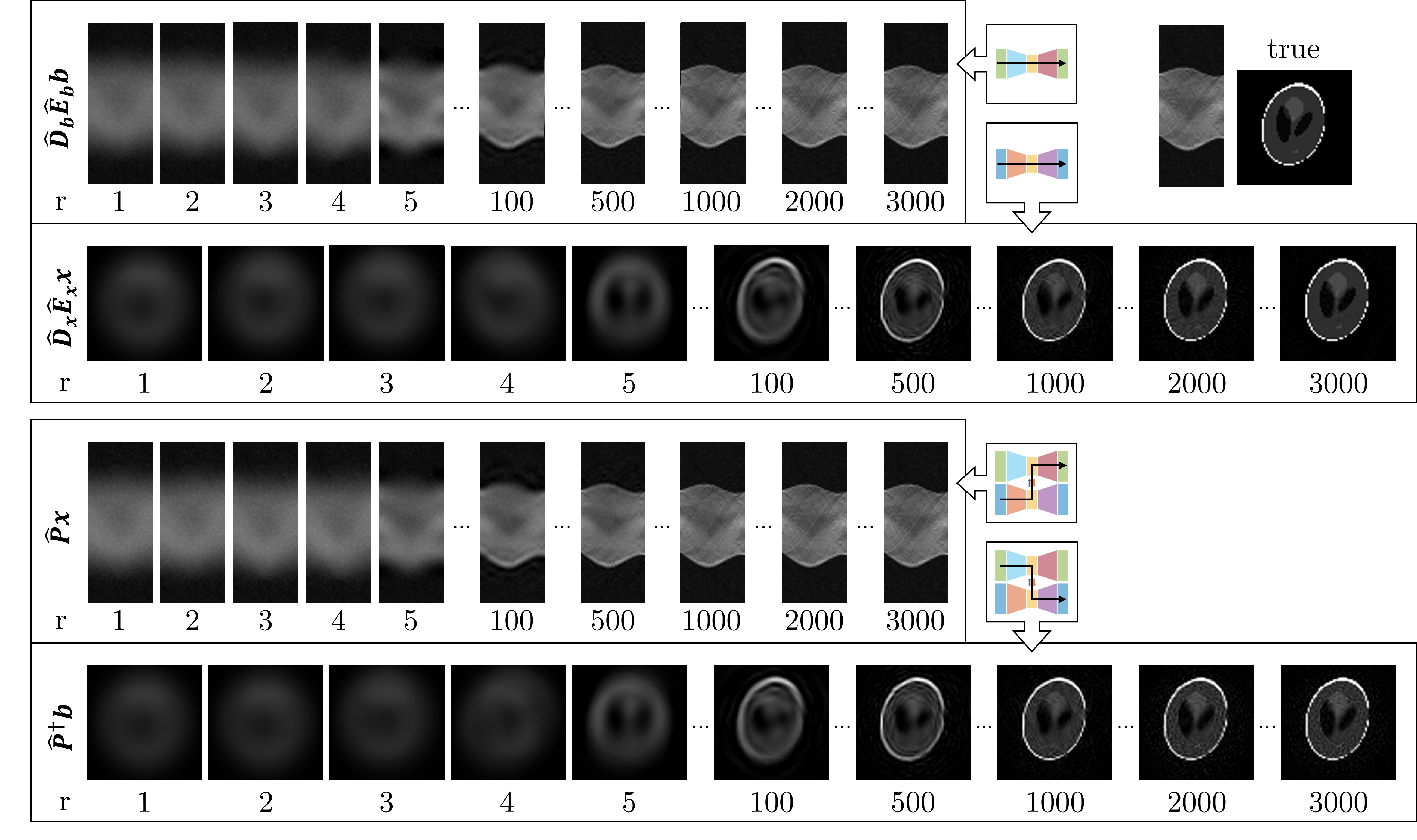}
    \caption{An illustration for one example image pair from the testing set for the CT example (true image shown in the top right corner). The top two rows contain reconstructed sinograms and phantoms from the autoencoders for $\bfb$ and $\bfx$ respectively, for different latent dimensions.  The bottom two rows contain reconstructions for the PAIR forward surrogate and the PAIR inverse surrogate respectively, for different latent dimensions.}
    \label{fig:latentexamples}
\end{figure}

\subsection{Nonlinear PAIR for image deblurring}
\label{sub:MNIST}
In this example, we illustrate PAIR with nonlinear CNN autoencoders and linear mappings between latent spaces.  We use the MNIST dataset \cite{deng2012mnist} of 28-by-28 pixel handwritten digits and consider $\bfA\bfx+\bfeps=\bfb$ where $\bfA$ represents blurring with a Gaussian kernel ($8\times8$ blur kernel with $\sigma=10$, so it resembles a box-car blur), $\bfeps$ is a realization of Gaussian white noise (with variance $0.01$), $\bfx$ is a sample from the original MNIST images, and $\bfb$ is a corrupted observation of $\bfx$.

\subsubsection{Model architecture, results, and comparisons}
\label{sub:MNIST_PAIR}
We begin with the self-supervised learning task of creating two nonlinear autoencoders, $\Phi_{\rm ae}^\bfb$ and $\Phi_{\rm ae}^\bfx$.  The first autoencoder is used for the input blurred MNIST images, $\bfb$, and the second autoencoder is used to represent the clear target images $\bfx$.  We split the data set into 50,000 training images, 10,000 validation images, and 10,000 testing images. We use the same architecture and methodologies for both the input and target autoencoders.  

For the autoencoder $\Phi_{\rm ae}^\bfx$,  we consider learning approaches to approximate the mapping $\bfx_{j}\to\bfx_{j}$ for $j=1,2,\ldots, 50,\!000$, designed with an hour-glass shape to learn some compressed representation of our original images.  For this application, we use a CNN consisting of five convolutional layers, with 2, 3, 3, 2, and 1 channel(s), respectively.  The first two convolutional layers make up the encoder, $e_\bfx$, and the last three make up the decoder, $d_\bfx$.  Each layer is padded and uses a $3\times 3$ kernel with a stride length equal to 1.  To create a reduced dimension latent space, we use max pooling in the encoder with a $2\times2$ pooling window.  In the decoder, we use upsampling with an upsampling factor of 2 for both the rows and columns.  Each layer uses a ReLU activation function, save the last layer of the decoder, which uses a sigmoid activation function.  The encoder includes 77 parameters and the decoder includes 159 for a total of 236 learnable parameters defining the autoencoder.  This architecture is intentionally kept simple to highlight the core concepts and facilitate understanding.  More sophisticated architectures could be used and may allow for even more compression and expressivity in the input/target latent spaces.  

To learn the parameters, we utilize the ADAM optimizer with a piecewise constant learning rate scheduler.  Both the input and target autoencoders use 400 epochs to train. The first hundred epochs use a learning rate of $10^{-3}$, the second hundred use $10^{-4}$, the third hundred use $10^{-3}$, and the last hundred use $10^{-4}$.  We learn in batches of 256 and use mean squared error to define the loss function. The linear latent forward map $\widehat\bfM$ and inverse map $\widehat\bfM^\dagger$ are obtained following \Cref{sub:empBayes latent maps} (\Cref{eq:empirical m and mdagger}).

\Cref{fig:MNISTautoencoderInversion} shows results for 10 reconstructions from inversion through the PAIR network, where the reconstruction of $\bfx$ from $\bfb$ is given as $\bfx_{\rm pred} = d_\bfx(\widehat\bfM^\dagger e_\bfb(\bfb))$. For comparison, the target/true image is provided along with the absolute pixel-wise error image. Although some finer details are lost and, in certain cases, ambiguous digits appear transformed, the PAIR inverse reconstructions demonstrate an ability to produce plausible solutions that effectively denoise and deblur the corrupted input images. For each test example, we compute the relative reconstruction error norm as $\text{rel}={||\bfx_{\rm pred}-\bfx||_2}/{||\bfx||_2}$, where  $\bfx_{\rm pred}$ is the predicted reconstruction and $\bfx$ is the true image.  Over the test set, the average relative reconstruction error is $0.3783$. 

\begin{figure}[htbp]
    \centering
    \includegraphics[width=0.9\linewidth]{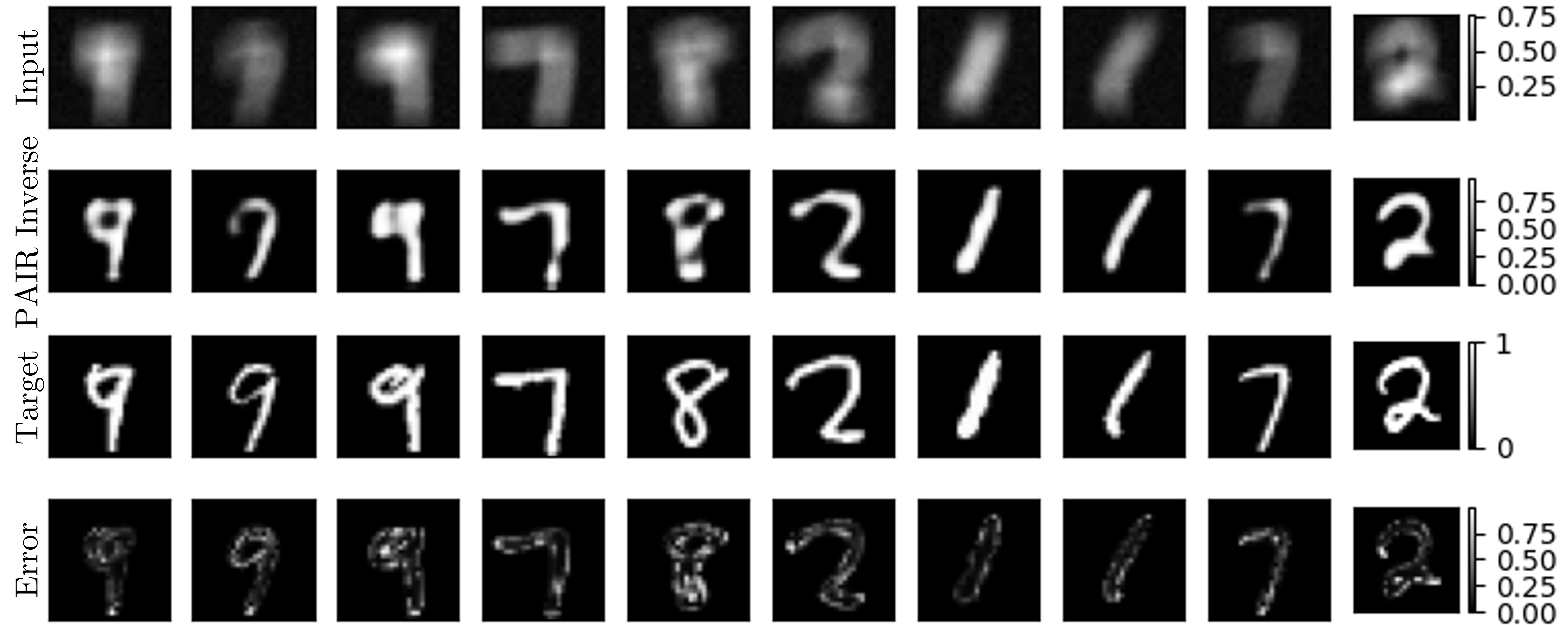}
    \caption{Example images inverted through the PAIR network.  For each sample from the test set, the top row shows a blurred input digit, the second row shows the predicted reconstruction, the third row shows the true original target image and the fourth row shows the absolute pixel-wise error between the true and predicted.}
    \label{fig:MNISTautoencoderInversion}
\end{figure}

Next, we provide a comparison of our PAIR inversion approach with an end-to-end approach, for different numbers of paired training samples. We use the same architecture of the encoder/decoder network to directly learn the input-output mapping $\widebar\Phi: \bfb_j \mapsto \bfx_j$ for all $j=1,2, \ldots, J$. For the nonlinear PAIR, the self-supervised learning task is conducted with all training images. The (linear) supervised learning task is conducted independently for each restricted number of supervised training samples.
For the direct inversion network, a first network is trained with 1,000 supervised samples for 400 epochs (100 with learning rate $10^{-3}$, 100 with $10^{-4}$.  For each 500 more supervised training samples allowed, the network is allowed to refine for 35 epochs (15 epochs with learning rate $10^{-3}$, 10 with $10^{-4}$).

In \Cref{fig:PAIRvsFull}, we provide the average relative reconstruction errors for the PAIR inverse reconstruction and the encoder-decoder inversion for varying $J$.  We can see that for large datasets (e.g., using all 60,000 training pairs), the end-to-end network slightly outperforms the PAIR network.  However, in the case that we have limited paired data available for the supervised learning task, the PAIR network can outperform the direct end-to-end approach by exploiting the abundance of unpaired samples for the self-supervised learning task.

\begin{figure}[htbp]
    \centering
    \includegraphics[width=0.7\linewidth]{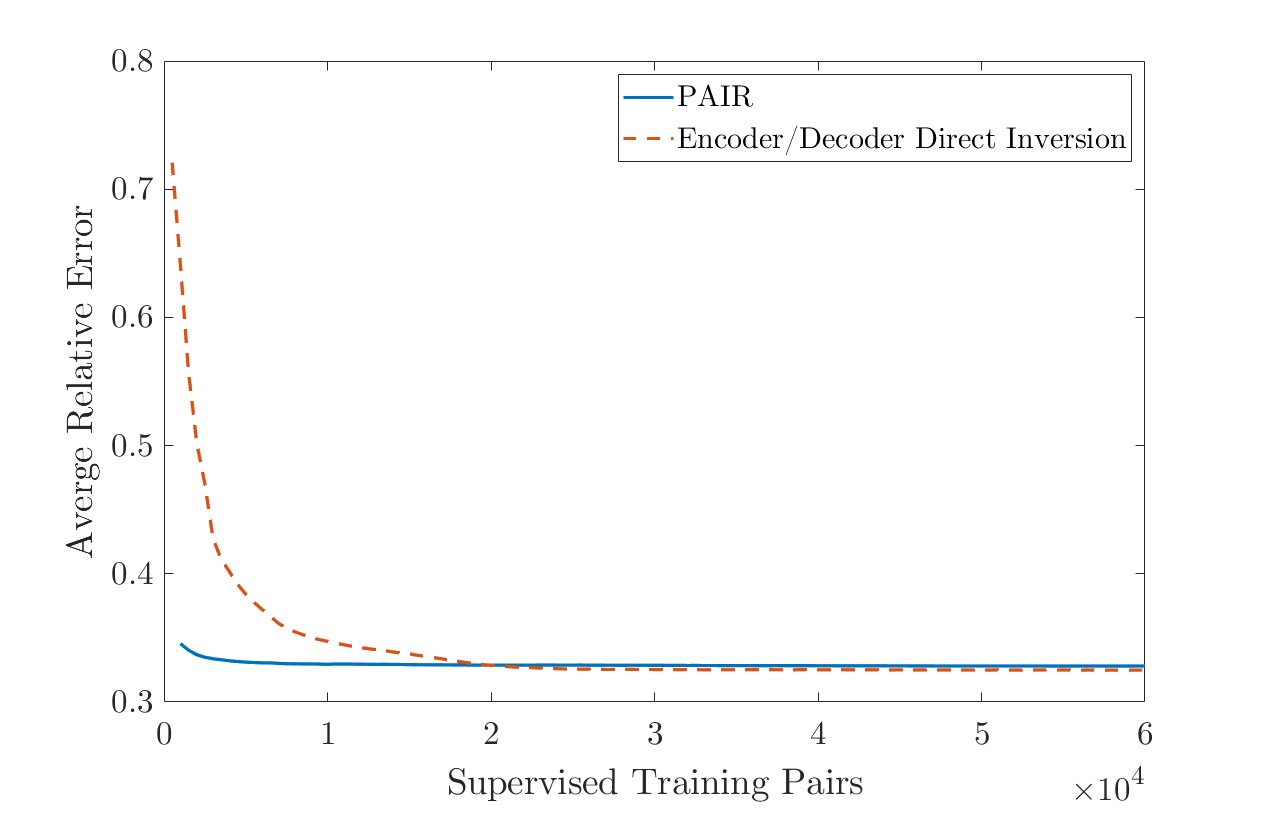}
    \caption{Average relative reconstruction error norms for the testing data for the nonlinear PAIR reconstruction and the direct encoder/decoder inversion network, for various numbers of supervised training pairs.  
    }
    \label{fig:PAIRvsFull}
\end{figure}

\subsubsection{Evaluation metrics from PAIR}
In the previous examples, we used the relative reconstruction error norm to evaluate the quality of a reconstruction, but since the true solution is not available in practice, other metrics are needed.  For example, a standard approach to determine if a reconstruction $\bfx_{\rm pred}$ is a plausible solution is to compute the residual,
\begin{equation}
    \bfr = A(\bfx_{\rm pred}) - \bfb.
\end{equation}
If the norm of $\bfr$ is sufficiently small, then the data is fitted well.  However, in practice, the forward model may not be available, or it may be prohibitively expensive to apply, making the residual norm impractical and/or impossible to compute.  

The PAIR framework can address this limitation by providing several cheaply computable metrics that can indicate the quality of a solution.  We consider five PAIR evaluation metrics:
\begin{itemize}
    \item the relative difference between the original and autoencoded observation,
    \begin{equation*}
        \frac{|| (d_\bfb \circ e_\bfb)(\bfb) - \bfb||_2}{||\bfb||_2}
    \end{equation*}
    \item the relative difference between the original and autoencoded prediction, 
    \begin{equation*}
        \frac{||(d_\bfx \circ e_\bfx)(\bfx_{\rm pred})||_2}{||\bfx_{\rm pred}||_2}
    \end{equation*}
    \item the relative residual estimate, 
    \begin{equation*}
        \frac{||d_\bfb(\widehat\bfM e_\bfx(\bfx_{\rm pred})) - \bfb ||_2}{||\bfb||_2}
    \end{equation*}
    \item the relative difference in the latent data space, 
    \begin{equation*}
        \frac{||\widehat\bfM^\dagger e_\bfb(\bfb) - e_\bfx(\bfx_{\rm pred})||_2}{||e_\bfx(\bfx_{\rm pred})||_2}
    \end{equation*}
    \item the relative difference in the latent parameter space,
    \begin{equation*}
        \frac{||\widehat\bfM e_\bfx(\bfx_{\rm pred}) - e_\bfb(\bfb)||_2}{||e_\bfb(\bfb)||_2}
    \end{equation*}
\end{itemize}
All of these PAIR metrics can be computed for the training data, providing a baseline distribution for comparison.  
Similar to \cite{chung2024paired}, we note that PAIR does not yield a probability density.  However, for a new observation $\bfb$ and predicted reconstruction $\bfx_{\rm pred}$, we can use the PAIR metrics for out-of-distribution (OOD) detection.  If the PAIR metrics for the new sample lie within high-probability regions, we have some indication we can trust the prediction from the PAIR network.  If the metrics for the new sample lie in a low probability region, further investigation is required.

Next, we provide an empirical investigation of the PAIR metrics for the MNIST PAIR described above, where we consider PAIR reconstructions for in-distribution and out-of-distribution data.  The MNIST PAIR network we have trained is  small, which may mean that it has limited generalizability (i.e., for ``similar enough" images, it may be sufficient). We will use the notMNIST \cite{bulatov2011notmnist} dataset of typed letters for illustration.  The MNIST and notMNIST images share many structural similarities: they are the same size, they are greyscale, and they contain a single contiguous white shape. Further, we use the same blurring operator and noise level.  

\begin{figure}[tb]
    \centering
    \includegraphics[width=0.8\linewidth]{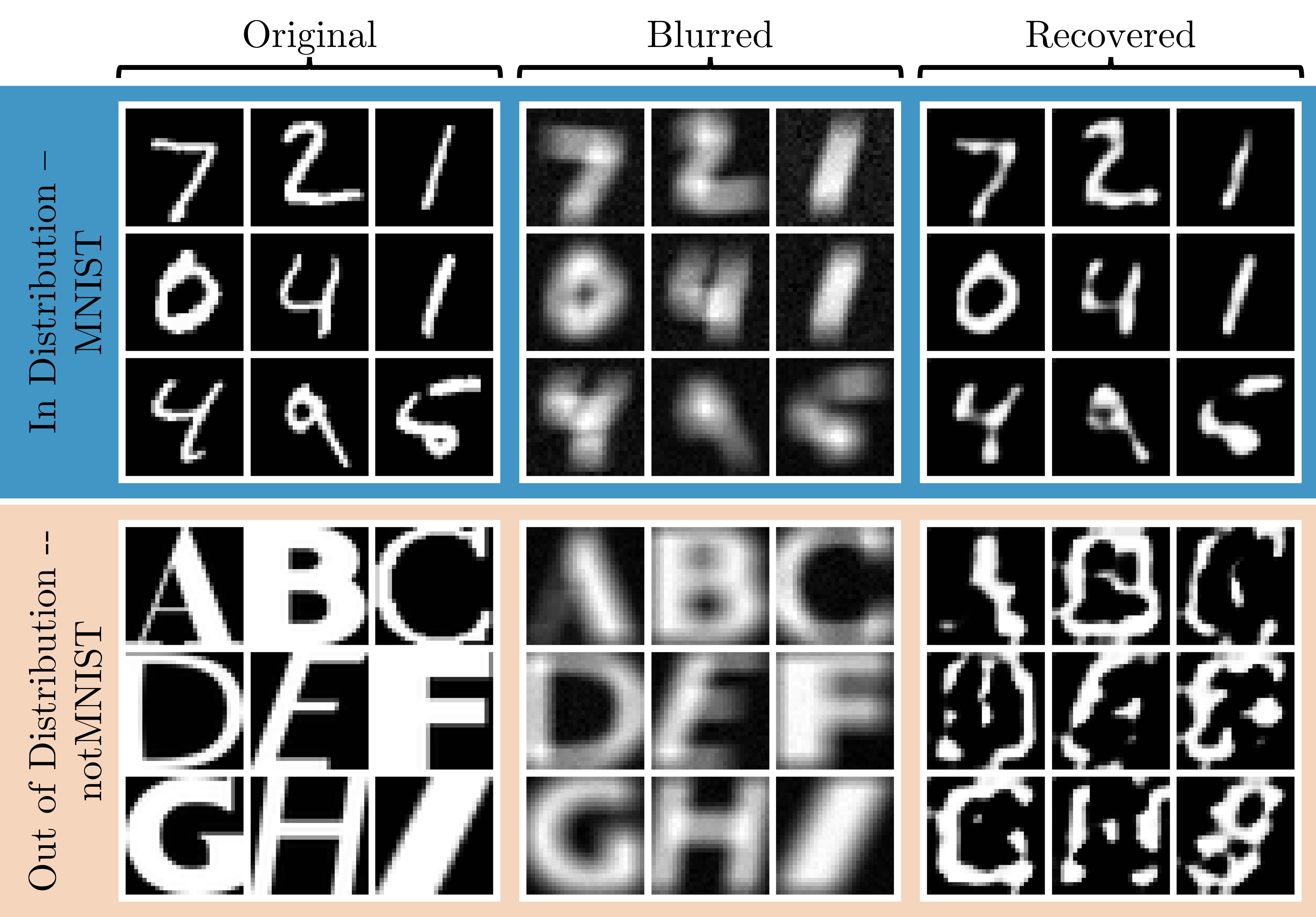}
    \caption{Examples of original, blurred, and recovered images for in-distribution samples (top row, from MNIST test set) and out-of-distribution samples (bottom row, from notMNIST dataset).  Images are recovered with the MNIST PAIR described in \Cref{sub:MNIST_PAIR}.}
    \label{fig:OODexamples}
\end{figure}

In \Cref{fig:OODexamples}, we provide some examples of original, blurred, and recovered images from both the in-distribution MNIST images and the OOD notMNIST images.  
Here it is clear that PAIR does not provide good reconstructions for the OOD images, and this is further confirmed by the PAIR metrics presented in \Cref{fig:OODhist}, where the distributions of each of the five metrics for in- and out-of-distribution samples are overlayed. For each of the PAIR metrics, the in-distribution samples (dark blue) come from the training set of MNIST, while out-of-distribution samples (light orange) come from notMNIST.  The last two metrics, where the difference is calculated in the autoencoder latent spaces, provide a clear separation between the two distributions. This is further confirmed in \Cref{fig:latent_metrics} where the two latent space metrics (that seem the most helpful for our exploration) are displayed as two separated densities for the in-distribution MNIST samples (dark blue) and OOD notMNIST samples. In summary, the PAIR metrics can help us to decide if a new image is ``similar enough" to the distribution on which the network was trained, thereby indicating whether it can be expected to have a good reconstruction.

\begin{figure}[htbp]
    \begin{tabular}{ccccc}
         \includegraphics[width=0.19\linewidth]{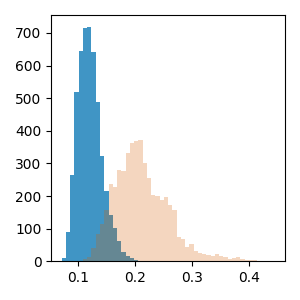} & \hspace{-0.5cm}
         \includegraphics[width=0.19\linewidth]{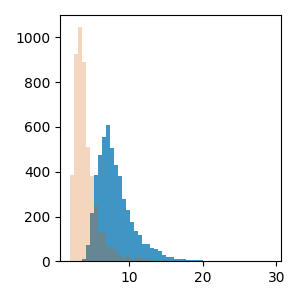} & \hspace{-0.5cm}
         \includegraphics[width=0.19\linewidth]{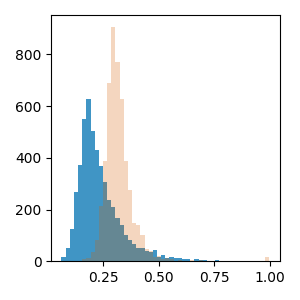} & \hspace{-0.5cm}
         \includegraphics[width=0.19\linewidth]{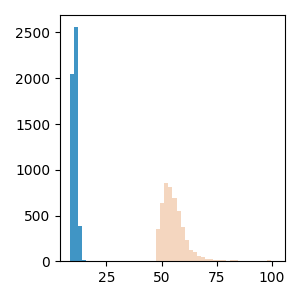} & \hspace{-0.5cm}
         \includegraphics[width=0.19\linewidth]{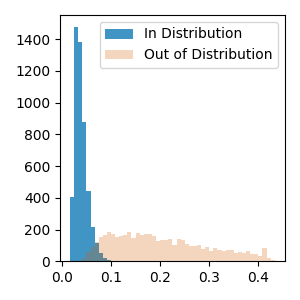} \\
         \scalebox{.65}
         {$\frac{\|(d_\bfb \circ e_\bfb)(\bfb)\|_2}{\|\bfb \|_2}$} &
         \scalebox{.65}
         {$\frac{\|d_\bfb(\bfM e_\bfx(\bfx_{\rm pred})) - \bfb\|_2}{\|\bfb\|_2}$} &
         \scalebox{.65}
         {$\frac{\|(d_\bfx \circ e_\bfx)(\bfx_{\rm pred}) - \bfx_{\rm pred}\|_2}{\|\bfx_{\rm pred} \|_2}$} &
         \scalebox{.65}
         {$\frac{\|\bfM^\dagger e_\bfb(\bfb) - e_\bfx(\bfx_{\rm pred}) \|_2}{\|e_\bfx(\bfx_{\rm pred}) \|_2}$} &
         \scalebox{.65}
         {$\frac{\|\bfM e_\bfx(\bfx_{\rm pred}) - e_\bfb(\bfb) \|_2}{\|e_\bfb(\bfb)\|_2}$}
    \end{tabular}
    \caption{PAIR metrics for indicating out-of-distribution.  In-distribution samples from the MNIST test set are shown in dark blue, while OOD notMNIST samples are shown in light orange (5000 samples each).
    }
    \label{fig:OODhist}
\end{figure}

\begin{figure}[htbp]
\centering
\begin{tikzpicture}
        \node[anchor=south west,inner sep=0] (image) at (0,0) {\includegraphics[width=0.7\textwidth]{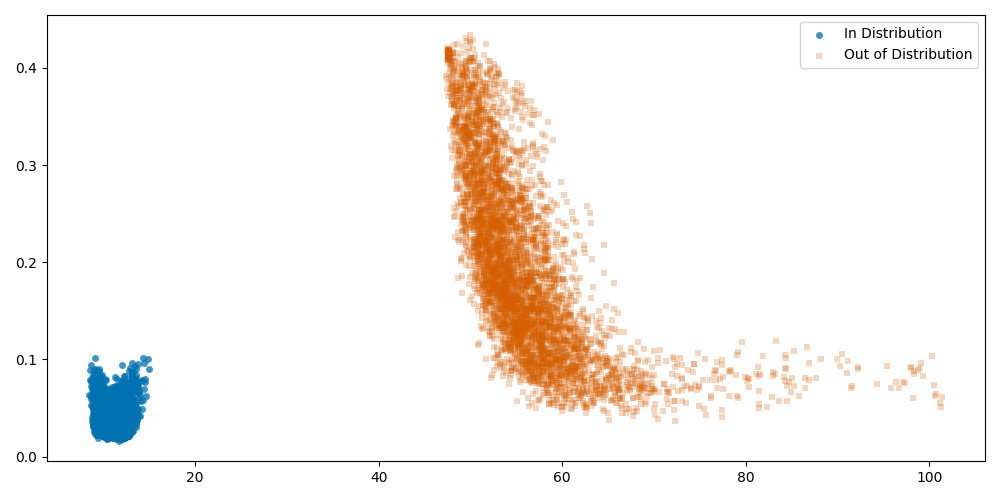}};
        \begin{scope}[x={(image.south east)},y={(image.north west)}]
            \node[rotate=90] at (-0.08, 0.5) {\textbf{$\frac{\|\bfM^\dagger e_\bfb(\bfb) - e_\bfx(\bfx_{\rm pred}) \|_2}{\|e_\bfx(\bfx_{\rm pred}) \|_2}$}}; 
            \node at (0.5, -0.1) {\textbf{$\frac{\|\bfM^\dagger e_\bfb(\bfb) - e_\bfx(\bfx_{\rm pred}) \|_2}{\|e_\bfx(\bfx_{\rm pred}) \|_2}$}}; 
        \end{scope}
    \end{tikzpicture}
    \caption{Scatterplot of samples from within distribution (MNIST) and outside of distribution (notMNIST) with the two most informative metrics explored earlier: the relative differences in the latent data and parameter spaces.}
    \label{fig:latent_metrics}
\end{figure}

\section{Conclusions}
\label{sec:conclusions}
In this work, we describe a paired autoencoder framework for inverse problems, where forward and inverse surrogates are obtained by combining two separate autoencoders for the input and target spaces with an optimal mapping between latent spaces. By exploiting Bayes risk and empirical Bayes risk minimization interpretations, theoretical results for PAIR are provided, with connections to existing works on low-rank matrix approximations and possible extensions to include other error metrics. Compared to end-to-end networks for inverse problems, the PAIR framework has the following potential benefits. First, PAIR uses self-supervised training for the autoencoders, so these can be done in parallel. Moreover, the number of training data for the input and the target space can be different (e.g., in medical imaging, datasets corresponding to observed sinograms may be significantly larger than datasets of true images/phantoms).  This allows the network to take full advantage of all available data. Second, supervised training of the forward/inverse model is only performed between the latent spaces.  Thus, for different forward models or a new application, the same target images (e.g., outputs) and the corresponding autoencoder can be reused.  Third, PAIR provides cheap metrics that can indicate whether or not a new sample is within the distribution on which the networks were trained.  

There are many potential applications and extensions of this work. 
By separating dimension reduction and inversion/forward propagation, we may independently address uncertainties arising from model imperfections, data noise, and data compression.
Future work includes using the PAIR framework to approximate adjoints for problems where the adjoint is too computationally intensive or impossible to compute (e.g., in the context of inexact Krylov methods).  The PAIR framework can also be used to define new data-driven priors (e.g., to approximate the mean and prior covariance matrix) and to investigate the uncertainty in solutions.

\bibliographystyle{siamplain}
\bibliography{references}
\end{document}

%% file: shared.tex
\usepackage{lipsum}
\usepackage{amsfonts}
\usepackage{graphicx}
\usepackage{epstopdf}
\usepackage{algorithmic}
\usepackage{pgfplots}
\ifpdf
  \DeclareGraphicsExtensions{.eps,.pdf,.png,.jpg}
\else
  \DeclareGraphicsExtensions{.eps}
\fi

\newsiamremark{remark}{Remark}
\newsiamremark{hypothesis}{Hypothesis}
\crefname{hypothesis}{Hypothesis}{Hypotheses}
\newsiamthm{claim}{Claim}

\headers{Paired Autoencoders}{E. Hart, J. Chung, and M. Chung}

\title{A Paired Autoencoder Framework for Inverse Problems via Bayes Risk Minimization
\thanks{Submitted to the editors January 24, 2025.
\funding{This material is based upon work supported by the U.S. Department of Energy, Office of Science, Office of Advanced Scientific Computing Research, Department of
Energy Computational Science Graduate Fellowship under Award Number DE-SC0024386. This work was partially supported by the National Science Foundation program under grants DMS-2152661 for M. Chung and DMS-2411197 for J. Chung. Any opinions, findings, and conclusions or recommendations expressed in this material are those of the author(s) and do not necessarily reflect the views of the National Science Foundation.}}}

\author{Emma Hart\thanks{Department of Mathematics, Emory University, Atlanta, GA 
  (\email{ehart5@emory.edu}, \email{jmchung@emory.edu}, \email{matthias.chung@emory.edu}).}
\and Julianne Chung\footnotemark[2]
\and Matthias Chung\footnotemark[2]}

\usepackage{amsopn}

%% file: math_commands.tex
\newcommand{\bfGamma}{{\boldsymbol{\Gamma}}}

\newcommand{\bfSigma}{{\boldsymbol{\Sigma}}}

\newcommand{\bfvarepsilon}{{\boldsymbol{\varepsilon}}}
\newcommand{\bfeps}{{\boldsymbol{\varepsilon}}}

\newcommand{\bftheta}{{\boldsymbol{\theta}}}

\newcommand{\bfA}{{\bf A}}
\newcommand{\bfB}{{\bf B}}

\newcommand{\bfD}{{\bf D}}
\newcommand{\bfE}{{\bf E}}

\newcommand{\bfI}{{\bf I}}

\newcommand{\bfK}{{\bf K}}
\newcommand{\bfL}{{\bf L}}
\newcommand{\bfM}{{\bf M}}

\newcommand{\bfP}{{\bf P}}

\newcommand{\bfU}{{\bf U}}
\newcommand{\bfV}{{\bf V}}

\newcommand{\bfX}{{\bf X}}
\newcommand{\bfY}{{\bf Y}}
\newcommand{\bfZ}{{\bf Z}}

\newcommand{\bfb}{{\bf b}}

\newcommand{\bfm}{{\bf m}}

\newcommand{\bfq}{{\bf q}}
\newcommand{\bfr}{{\bf r}}

\newcommand{\bfu}{{\bf u}}
\newcommand{\bfv}{{\bf v}}

\newcommand{\bfx}{{\bf x}}

\newcommand{\bfz}{{\bf z}}
\newcommand{\bfzero}{{\bf0}}

\newcommand{\bbE}{\mathbb{E}}

\newcommand{\bbR}{\mathbb{R}}

\newcommand{\calL}{\mathcal{L}}

\newcommand{\calZ}{\mathcal{Z}}

\DeclareFontFamily{U}{mathx}{\hyphenchar\font45}
\DeclareFontShape{U}{mathx}{m}{n}{<-> mathx10}{}
\DeclareSymbolFont{mathx}{U}{mathx}{m}{n}
\DeclareMathAccent{\widebar}{0}{mathx}{"73}

\newcommand{\norm} [2][]{\left\|#2\right\|_{#1}}           
\newcommand{\mdot} {\,\cdot\,}                             
\newcommand{\vertbar}{\rule[-1ex]{0.5pt}{2.5ex}}           
\newcommand{\horzbar}{\rule[.5ex]{2.5ex}{0.5pt}}           
\DeclareMathOperator*{\argmin}{arg\,min}                   
\newcommand{\rank} [1]  {{\rm rank}\left( #1 \right)}    
\renewcommand{\t} {^{\top}}                                
\newcommand{\fro}   {{\rm F}}                              
\newcommand{\trace} [1]  {{\rm tr\!}\left( #1 \right)}     
\newcommand{\eps} {\varepsilon}                            
\renewcommand{\vec}[1] {{\rm vec}(#1)}  

\definecolor{myorange}{RGB}{232,119,34} 
\definecolor{ocre}{RGB}{232,119,34} 
\definecolor{mygreen}{RGB}{28,172,0} 
\definecolor{myblack}{RGB}{12,12,12} 
\definecolor{maroon}{rgb}{0.5, 0.0, 0.0}
\definecolor{orange}{RGB}{232,119,34}
\definecolor{matlab1}{RGB}{0,    114,  189} 
\definecolor{matlab2}{RGB}{217,   83,   25}
\definecolor{matlab3}{RGB}{237,  177,   32}
\definecolor{matlab4}{RGB}{126,   47,  142}
\definecolor{matlab5}{RGB}{119,  172,   48}
\definecolor{matlab6}{RGB}{77,   190,  238}
\definecolor{matlab7}{RGB}{162,   20,   47}
\colorlet{designcolor}{matlab1}
\colorlet{ocre}{matlab1} 
\colorlet{myorange}{matlab1} 